\documentclass[11pt]{article}

\usepackage[top=50mm, bottom=50mm, left=50mm, right=50mm]{geometry}
\usepackage{lineno}
\usepackage{amssymb}
\usepackage{amsmath}
\usepackage{amsthm}
\usepackage{epsfig}
\usepackage{graphicx}
\usepackage{graphics}
\usepackage{float}
\usepackage{multirow}
\usepackage{color}
\usepackage{lineno}
\usepackage{fullpage}
\usepackage[normalem]{ulem} 
\usepackage{makeidx}
\usepackage{xspace}
\usepackage{wrapfig}
\makeindex

\newtheorem{theorem}{Theorem}

\newtheorem{lemma}{Lemma}

\definecolor{darkred}{rgb}{1, 0.1, 0.3}
\definecolor{darkblue}{rgb}{0.1, 0.1, 1}
\definecolor{darkgreen}{rgb}{0,0.6,0.5}

\newcommand {\mm}[1] {\ifmmode{#1}\else{\mbox{\(#1\)}}\fi}

\usepackage{hyperref}
\hypersetup{colorlinks=true,allcolors=[rgb]{0,0.25,0.65}}

\usepackage{microtype}
\usepackage{graphicx}
\usepackage{booktabs} 
\usepackage[utf8]{inputenc}
\usepackage{graphicx}
\usepackage{amsmath}
\DeclareMathOperator*{\argmin}{argmin}
\usepackage{amsthm}
\usepackage{amssymb}

\newtheorem{proposition}{Proposition}
\newtheorem{assumption}{Assumption}

\newtheorem{definition}{Definition}
\usepackage{mdframed}
\usepackage{lipsum}
\usepackage{url}
\usepackage{caption}
\usepackage{subcaption}
\newmdtheoremenv{theo}{Theorem}
\newsavebox{\savepar}

\usepackage{bm}
\usepackage{enumitem}
\setitemize{noitemsep,topsep=0pt,parsep=0pt,partopsep=0pt}
\usepackage[ruled, vlined]{algorithm2e}
\begin{document}

\title{Linear Regression Games: Convergence Guarantees to Approximate Out-of-Distribution Solutions}
 
\author{Kartik Ahuja$^{*}$ \and Karthikeyan Shanmugam$^{*}$ \and Amit Dhurandhar\footnote{ IBM Research, Thomas J. Watson Research Center, Yorktown Heights, New York }}
\date{}

\maketitle

\begin{abstract}
Recently, invariant risk minimization (IRM) \cite{arjovsky2019invariant} was proposed as a promising solution to address out-of-distribution (OOD) generalization. In \cite{ahuja2020invariant}, it was shown that solving for the Nash equilibria of a new class of “ensemble-games” is equivalent to solving IRM. In this work, we extend the framework in \cite{ahuja2020invariant} for linear regressions by projecting the ensemble-game on an $\ell_{\infty}$ ball. We show that such projections help achieve non-trivial OOD guarantees despite not achieving perfect invariance.  For linear models with confounders, we prove that Nash equilibria of these games are closer to  the ideal OOD solutions than the standard empirical risk minimization (ERM) and we also provide learning algorithms that provably converge to these Nash Equilibria. Empirical comparisons of the proposed approach with the state-of-the-art show consistent gains in achieving OOD solutions in several settings involving anti-causal variables and confounders.
\end{abstract}

\section{Introduction}
A recent study shows that models trained to detect COVID-19 from chest radiographs rely on spurious factors such as the source of the data rather than the lung pathology \cite{degrave2020ai}. This is just one of many alarming examples of spurious correlations failing to hold outside a  specific training distribution.  In one commonly cited example, \cite{beery2018recognition} trained a convolutional neural network (CNN) to classify camels from cows. In the training dataset,  most pictures of  the cows had green pastures, while most pictures of camels were in the desert. The CNN picked up the spurious correlation and associated green pastures with cows thus failing to classify cows on beaches correctly.

Recently, \cite{arjovsky2019invariant} proposed a framework called invariant risk minimization (IRM) to address the problem of models inheriting spurious correlations. They showed that when data is gathered from multiple environments, one can learn to exploit invariant causal relationships, rather than relying on varying spurious relationships, thus learning robust predictors.  The authors used the invariance principle based on causality \cite{pearl1995causal} to construct powerful objects called ``invariant predictors''. An invariant predictor loosely speaking is a predictor that is simultaneously optimal across all the training environments under a shared representation.  In \cite{arjovsky2019invariant}, it was shown that for linear models with confounders and/or anti-causal variables, learning ideal invariant predictors translates to learning solutions with ideal out-of-distribution (OOD) generalization behavior. However, building efficient algorithms guaranteed to learn these invariant predictors is still a challenge.

 The algorithm in \cite{arjovsky2019invariant} is based on minimizing a risk function comprising of the standard risk and a penalty term that tries to approximately ensure that predictors learned are invariant. The penalty is non-convex even for linear models and thus the algorithm is not guaranteed to arrive at invariant predictors. Another recent work \cite{ahuja2020invariant}, proposed a framework called invariant risk minimization games (IRM-games) and showed that solving for the Nash equilibria (NE) of a special class of ``ensemble-games'' is equivalent to solving IRM for many settings. The algorithm in \cite{ahuja2020invariant} has no convergence guarantees to the NE of the ensemble-game. To summarize,  building algorithms that are guaranteed to converge to predictors with non-trivial OOD generalization is unsolved even for linear models with confounders and/or anti-causal variables. 
 
 In this work, we take important steps towards this highly sought after goal. As such, we formulate an ensemble-game that is constrained to be in the $\ell_{\infty}$ ball. 
 Although this construction might seem to be surprising at first, we show that these constrained ensemble-game based predictors have a good OOD behavior even though that they may not be the exact invariant predictors. We provide efficient algorithms that are guaranteed to learn these predictors in many settings. To the best of our knowledge, our algorithms are the first for which we can guarantee both convergence and better OOD behavior than standard empirical risk minimization. We carry out empirical comparisons in the settings proposed in \cite{arjovsky2019invariant}, where the data is generated from models that include both causal and anti-causal variables as well as confounders in some cases. These comparisons of our approach with the state-of-the-art depict  its promise in achieving OOD solutions in these setups. This demonstrates that searching over the NE of constrained ensemble-games is a principled alternative to searching over invariant predictors as is done in IRM.

\section{Related Work}

 IRM \cite{arjovsky2019invariant} has its roots in the theory of causality \cite{pearl1995causal}. A variable $y$ is caused by a set of non-spurious actual causal factors $x_{\mathrm{Pa}(y)}$ if and only if in all environments where $y$ has not been intervened on, the conditional probability $P(y|x_{\mathrm{Pa}(y)})$ remains invariant. This is called the \textit{modularity condition} \cite{bareinboim2012local}.
Related and similar notions are the \emph{independent causal mechanism principle} \cite{scholkopf2012causal}\cite{janzing2010causal}\cite{janzing2012information} and the \emph{invariant causal prediction principle} (ICP) \cite{peters2016causal}\cite{heinze2018invariant}. These principles imply that if all the environments (train and test) are modeled by interventions that do not affect the causal mechanism of target variable $y$, then a predictor trained on the transformation that involves the causal factors  ($\Phi(x)= x_{\mathrm{Pa}(y)}$) to predict $y$ is an invariant predictor, which is robust to unseen interventions.

In general, for finite sets of environments, there may be other invariant predictors. If one has information about the causal Bayesian network structure, one can find invariant predictors that are maximally predictive  using conditional independence tests and other graph-theoretic tools \cite{magliacane2018domain,subbaswamy2019should}. The above works select subsets of features, primarily using conditional independence tests, that make the optimal predictor trained on the selected features invariant. 
In IRM \cite{arjovsky2019invariant}, the authors provide an optimization-based reformulation of this invariance that facilitates searching over transformations in a continuous space. Following the original work  IRM from \cite{arjovsky2019invariant}, there have been several interesting works ---  \cite{teney2020unshuffling}\cite{krueger2020out}\cite{chang2020invariant}\cite{koyama2020out}
 is an incomplete representative list --- that build new methods inspired from IRM to address OOD generalization.  In these works, similar to IRM, the algorithms are not provably guaranteed to converge to predictors with  desirable OOD behavior. 
 
\section{Background, Problem Formulation, and Approach}

\subsection{Nash Equilibrium and Concave Games}

A standard normal form game is written as a tuple $\Omega = (\mathcal{N}, \{u_i\}_{i \in \mathcal{N}},\{\mathcal{S}_i\}_{i\in \mathcal{N}})$, where $\mathcal{N}$ is a finite set of players.  Player $i \in \mathcal{N}$ takes actions from a strategy set $\mathcal{S}_i$. The utility of player $i$ is $u_i:\mathcal{S} \rightarrow \mathbb{R}$, where we write the joint set of actions of all the players as  $\mathcal{S} = \Pi_{i\in \mathcal{N}} \mathcal{S}_i$. The joint strategy of all the players is given as $\bm{s} \in \mathcal{S}$,  the  strategy of player $i$ is $\bm{s}_i$ and the strategy of the rest of players is $\bm{s}_{-i} = (\bm{s}_{i^{'}})_{i^{'} \not = i}$.

\begin{definition}
A strategy $\bm{s}^{\dagger}\in \mathcal{S}$ is said to be a pure strategy Nash equilibrium (NE) if it satisfies
$$u_i(\bm{s}_{i}^{\dagger},\bm{s}_{-i}^{\dagger}) \geq u_i(k,\bm{s}_{-i}^{\dagger}), \forall k \in \mathcal{S}_{i}, \forall i \in \mathcal{N}$$
\end{definition}
  
NE defines a state where each player is using the best possible strategy in response to the rest of the players. A natural question to ask is when does a pure strategy NE exist. In the seminal work of \cite{debreu1952social} it was shown that for a special class of games called concave games such a NE always exists. 

\begin{definition}
A game $\Omega$ is called a concave game if for each $i\in \mathcal{S}$
\begin{itemize}
    \item  $\mathcal{S}_i$ is a compact, convex subset of $\mathbb{R}^{m_i}$
    \item $u_{i}(\bm{s}_i,\bm{s}_{-i})$ is continuous in $\bm{s}_{-i}$
    \item $u_{i}(\bm{s}_i,\bm{s}_{-i})$ is continuous and concave in $\bm{s}_i$ .
\end{itemize}
\label{assm1:Existence}
\end{definition}
\begin{theorem} \label{thm1} \cite{debreu1952social}
For any concave game $\Omega$  a pure strategy Nash equilibrium $\bm{s}^{\dagger}$ always exists.
\end{theorem}
In this work, we only study  pure strategy NE and use the terms pure strategy NE and NE interchangeably.

\subsection{Invariant Risk Minimization \& Invariant Risk Minimization Games} 
\label{secn: IRM}
We are given a collection of training datasets $D = \{D_e\}_{e\in \mathcal{E}_{tr}}$ gathered from a set of environments $\mathcal{E}_{tr}$, where $D_e=\{\bm{x}^{i}_e, y^{i}_e\}_{i=1}^{n_e}$ is the dataset gathered from environment $e\in \mathcal{E}_{tr}$ and $n_e$ is the number of points in environment $e$. 
The feature value for data point $i$ is $\bm{x}_e^{i} \in \mathcal{X}$ and the corresponding label is $y_e^{i}\in \mathcal{Y}$, where $\mathcal{X} \subseteq \mathbb{R}^{n}$ and $\mathcal{Y}\subseteq \mathbb{R}$. Each point $(\bm{x}_e^{i},y_e^{i})$ in environment $e$ is drawn i.i.d from  a distribution $\mathbb{P}_e$.
Define a predictor $f:\mathcal{X} \rightarrow \mathbb{R}$. The goal of IRM is to use these collection of datasets $D$ to construct a predictor $f$ that performs well across many unseen environments $\mathcal{E}_{all}$, where $ \mathcal{E}_{all} \supseteq \mathcal{E}_{tr}$. Define the risk achieved by $f$ in environment $e$ as $R_e(f) = \mathbb{E}_{e}\big[\ell(f(\bm{X}_e), Y_e)\big]$, where $\ell$ is the square loss when $f(\bm{X}_{e})$ is the predicted value and $Y_{e}$ is the corresponding label, $(\bm{X}_e,Y_e)\sim \mathbb{P}_e$ and the expectation $\mathbb{E}_{e}$ is defined with respect to (w.r.t.) the distribution of points in environment $e$.

\textbf{Invariant predictor and IRM optimization:} An invariant predictor is composed of two parts a representation $\bm{\Phi} \in \mathbb{R}^{ d \times n}$ and a predictor $\bm{w} \in \mathbb{R}^{d\times 1}$. We say that a data representation $\bm{\Phi}$ elicits an invariant predictor $\bm{w}^{\mathsf{T}}\bm{\Phi}$ across the set of environments $\mathcal{E}_{tr}$ if there is a predictor $\bm{w}$ that achieves the minimum risk for all the environments 
$\bm{w} \in \argmin_{\tilde{\bm{w}} \in \mathbb{R}^{d\times 1}} R_{e}(\tilde{\bm{w}}^{\mathsf{T}}\bm{\Phi}), \; \forall e \in \mathcal{E}_{tr}$.  IRM may be phrased as the following constrained optimization problem:
\begin{equation}
    \begin{split}
        & \min_{\bm{\Phi} \in \mathbb{R}^{d\times n},\bm{w}  \in \mathbb{R}^{d\times 1}} \sum_{e \in \mathcal{E}_{tr}}R_{e}(\bm{w}^{\mathsf{T}}\bm{\Phi}) \\ 
        & \text{s.t.}\;\bm{w} \in \argmin_{\tilde{\bm{w}} \in \mathbb{R}^{d\times 1}} R_{e}(\tilde{\bm{w}}^{\mathsf{T}}\bm{\Phi}),\;\forall e \in \mathcal{E}_{tr}
    \end{split}
    \label{eqn: IRM}
\end{equation}
If $\bm{w}^{\mathsf{T}} \bm{\Phi}$ satisfies the constraints above, then it is an invariant  predictor across the training environments $\mathcal{E}_{tr}$.  Define the set of invariant predictors $\bm{w}^{\mathsf{T}} \bm{\Phi}$ satisfying the constraints in \eqref{eqn: IRM} as  $\mathcal{S}^{\mathsf{IV}}$. Informally stated, the main idea behind the above optimization is inspired from invariance principles in causality \cite{bareinboim2012local}\cite{pearl2009causality}. Each environment can be understood as an intervention. By learning an invariant predictor the learner hopes to identify a representation $\bm{\Phi}$ that transforms the observed features into the causal features and the optimal model trained on causal representations are likely to be same (invariant) across the environments provided we do not intervene on the label itself. These invariant models can be shown to have a good out-of-distribution performance. Next, we briefly describe IRM-games.

\textbf{Ensemble-game:} Each environment $e$ is endowed with its own predictor $\bm{w}_{e} \in \mathbb{R}^{d\times 1}$. Define  an ensemble predictor  $\bar{\bm{w}}\in \mathbb{R}^{d\times 1}$ given as $\bar{\bm{w}}= \sum_{q\in \mathcal{E}_{tr}}\bm{w}_{q}$; for the rest of this work a bar on top of vector represents an ensemble predictor.    We  require all the environments to use this ensemble $\bar{\bm{w}}$. We want to solve the following new optimization problem.
\begin{equation*}
    \begin{split}
        & \min_{\bm{\Phi} \in \mathbb{R}^{d\times n}, \bar{\bm{w}} \in \mathbb{R}^{d\times 1}} \sum_{e \in \mathcal{E}_{tr}}R_{e}\Big(\bar{\bm{w}}^{\mathsf{T}}\bm{\Phi}\Big)  \\
        & \text{s.t.}\;\bm{w}_e \in \argmin_{\tilde{\bm{w}}_{e} \in \mathbb{R}^{d\times 1}} R_{e}\bigg(\Big[\tilde{\bm{w}}_{e} + \sum_{q\in \mathcal{E}_{tr}\backslash \{e\}}\bm{w}_{q}\Big]^{\mathsf{T}}\bm{\Phi}\bigg),\;\forall e \in \mathcal{E}_{tr} \\ 
    \end{split}
    \label{eqn: IRM_ensemble}
\end{equation*}
For a fixed representation $\bm{\Phi}$, the constraints in the above optimization \eqref{eqn: IRM_ensemble} represent the NE of a game with each environment $e$ as a player with actions $\tilde{\bm{w}}_e$. Environment $e$ selects $\tilde{\bm{w}}_e$ to maximize its utility  $-R_{e}\bigg(\Big[\tilde{\bm{w}}_{e} + \sum_{q\not=e}\tilde{\bm{w}}_{q}\Big]^{\mathsf{T}}\bm{\Phi}\bigg)$.  Define the set of ensemble-game predictors $\bar{\bm{w}}^{\mathsf{T}}\bm{\Phi}$, i.e. the predictors that satisfy the constraints in \eqref{eqn: IRM_ensemble} as $\mathcal{S}^{\mathsf{EG}}$. In \cite{ahuja2020invariant} it was shown that the set of ensemble $\mathcal{S}^{\mathsf{EG}} = \mathcal{S}^{\mathsf{IV}}$. Having briefly reviewed IRM and  IRM-games (we presented them with linear models but these works are more general), we are now ready to build our framework.

\subsection{Unconstrained Linear Regression Games} 
\label{senc:ulrg}
The data is gathered from a set of two environments, $\mathcal{E}_{tr}= \{1,2\}$. \footnote{Discussion on multiple environments is in the Appendix Section.} Each data point $(\bm{X}_e,Y_e)$ in environment $e$ is sampled from $\mathbb{P}_e$. Each environment $e \in \{1,2\}$ is a player that wants to select a $\bm{w}_e\in \mathbb{R}^{n\times 1}$ such that it minimizes
\begin{equation}
    R_{e}(\bm{w}_1,\bm{w}_2) = \mathbb{E}_{e}\Big[\big(Y_{e} - \bm{w}_{1}^{\mathsf{T}}\bm{X}_{e} - \bm{w}_{2}^{\mathsf{T}}\bm{X}_{e}\big)^2\Big]
\end{equation}
where $\mathbb{E}_e$ is expectation w.r.t $\mathbb{P}_e$. We write the above as a two player game represented by a tuple $\Gamma = (\{1,2\}, \{R_{e}\}_{e\in \{1,2\}}, \mathbb{R}^{n\times 1})$. We refer to $\Gamma$ as a unconstrained linear regression game (U-LRG). 
A Nash equilibrium $\bm{w}^{\dagger} = (\bm{w}_1^{\dagger},\bm{w}_2^{\dagger})$ of U-LRG is a solution to 
\vspace{-0.5em}
\begin{equation}
\begin{split}
    & \bm{w}_1^{\dagger} \in \argmin_{\tilde{\bm{w}}_1 \in \mathbb{R}^{n \times 1}} \mathbb{E}_{1}\Big[\big(Y_{1} - \tilde{\bm{w}}_{1}^{\mathsf{T}}\bm{X}_{1} - \bm{w}_{2}^{\dagger,\mathsf{T}}\bm{X}_{1}\big)^2\Big]  \\ 
    & \bm{w}_2^{\dagger} \in \argmin_{\tilde{\bm{w}}_2 \in \mathbb{R}^{n \times 1}} \mathbb{E}_{2}\Big[\big(Y_{2} - \bm{w}_{1}^{\dagger,\mathsf{T}}\bm{X}_{2} - \tilde{\bm{w}}_{2}^{\mathsf{T}}\bm{X}_{2}\big)^2\Big] 
\end{split}
\end{equation}
The above two-player U-LRG is a natural extension of linear regressions and we start by analyzing the NE of the above game. Before going further, the above game can be understood as fixing $\bm{\Phi}$ to identity in the ensemble-game defined in the previous section.

For each $e\in \{1,2\}$, define the mean of features $\bm{\mu}_{e}= \mathbb{E}_{e}[\bm{X}_e]$,  $\bm{\Sigma}_{e} = \mathbb{E}_{e}\big[\bm{X}_{e}\bm{X}_e^{\mathsf{T}}\big]$ and the correlation between the feature $\bm{X}_{e}$ and  the label $Y_{e}$ as $\boldsymbol{\rho}_e = \mathbb{E}_{e}\big[\bm{X}_{e}Y_{e}\big]$. 

\begin{assumption}
\textbf{Regularity condition.} For each $e\in \{1,2\}$, $\bm{\mu}_{e} = \bm{0}$
and $\bm{\Sigma}_{e}$ is positive definite.
\label{ass1}
\end{assumption}

The above regularity conditions are fairly standard and the mean zero condition can be relaxed by introducing intercepts in the model. When $\bm{\mu}_{e} = \bm{0}$, $\bm{\Sigma}_{e}$ is the covariance matrix.   For each $e\in \{1,2\}$, define $\bm{w}_e^{*} = \bm{\Sigma}_e^{-1}\bm{\rho}_e$, where $\bm{\Sigma}_e^{-1}$ is the inverse of $\bm{\Sigma}_e$. $\bm{w}_e^{*}$ is the least squares optimal solution for environment $e$, i.e., it solves $\min_{\bm{\tilde{w}}\in \mathbb{R}^{n\times 1}}\mathbb{E}_{e}\Big[\big(Y_{e} - \bm{\tilde{w}}^{\mathsf{T}}\bm{X}_{e}\big)^2\Big]$.

\begin{proposition}
    If Assumption \ref{ass1} holds and  if the least squares optimal solution in the two environments are 
    \begin{itemize}
    \item equal, i.e., $\bm{w}_{1}^{*}=\bm{w}_2^{*}$, then the set $\{(\bm{w}_1^{\dagger}, \bm{w}_2^{\dagger}) \;| \; \bm{w}_1^{\dagger}+\bm{w}_2^{\dagger}= \bm{w}_1^{*}\}$ describes all the pure strategy Nash equilibrium of U-LRG, $\Gamma$.
    \item  not equal, i.e., $\bm{w}_{1}^{*}\not=\bm{w}_2^{*}$, then U-LRG, $\Gamma$, has no pure strategy Nash equilibrium.
    \end{itemize}
    \label{prop1}
\end{proposition}
The proofs to all the propositions and theorems are in the Appendix.
From the above proposition, it follows that agreement between the environments on least squares optimal solution is both necessary and sufficient for the existence of NE of U-LRG. Next, we describe the family of linear structural equation models (SEMs) in \cite{arjovsky2019invariant} and show how the two cases, $\bm{w}_{1}^{*}=\bm{w}_2^{*}$, and $\bm{w}_{1}^{*}\not=\bm{w}_2^{*}$ naturally arise. 

\subsubsection{Nash Equilibria for Linear SEMs}
In this section, we consider linear SEMs from \cite{arjovsky2019invariant} and study the NE of U-LRG.
\begin{assumption} \textbf{Linear SEM with confounders and (or) anti-causal variables}
For each $e\in \{1,2\}$, $(\bm{X}_e,Y_e)$ is generated from the following SEM
\label{linear_model_ass}
\begin{equation}
\begin{split}
    & Y_{e} \leftarrow \bm{\gamma}^{\mathsf{T}}\bm{X}_{e}^{1} + \bm{\eta}_{e}^{\mathsf{T}}\bm{H}_{e}+ \varepsilon_{e}, \\
    & \bm{X}^2_{e} \leftarrow \bm{\alpha}_e Y_{e} + \bm{\Theta}_e \bm{H}_e  + \bm{\zeta}_e
\end{split}
\label{linear_model}
\end{equation}

The  feature vector is $\bm{X}_e = (\bm{X}_e^1, \bm{X}_e^2)$. $\bm{H}_e\in \mathbb{R}^{s}$ is a confounding random variable, where each component of $\bm{H}_e$ is an i.i.d draw from a distribution with zero mean and unit variance. $\bm{H}_e$ affects both the labels $Y_e$ through weights $\bm{\eta}_e\in \mathbb{R}^s$ and a subset of features $\bm{X}_e^{2}\in \mathbb{R}^{q}$ through weights $\bm{\Theta}_e\in \mathbb{R}^{q\times s}$. $\varepsilon_e\in \mathbb{R}$ is independent zero mean noise in the label generation. $Y_e$ affects a subset of features $\bm{X}_e^2$  with weight $\bm{\alpha}_e\in \mathbb{R}^{q}$, $\bm{\zeta}_e\in \mathbb{R}^{q}$ is an independent zero mean noise vector affecting $\bm{X}_e^2$. $\bm{X}_e^1\in \mathbb{R}^{p}$ are the causal features drawn from a distribution with zero mean and affect the label through a weight $\bm{\gamma}\in \mathbb{R}^{p}$, which is invariant across the environments. 
\end{assumption}

The above model captures many different settings. If $\bm{\alpha}_e=\bm{0}$ and  $\bm{\Theta}_e\not=\bm{0}$, then features $\bm{X}_e^{2}$ appear correlated with the label due to the confounder $\bm{H}_e$.  If $\bm{\alpha_e}\not=\bm{0}$ and  $\bm{\Theta}_e=\bm{0}$, then  features $\bm{X}_e^{2}$ are correlated with the label but they are effects or anti-causal. If both $\bm{\alpha}_e\not=\bm{0}, \bm{\Theta}_e\not=\bm{0}$, then we are in a hybrid of the above two settings.  In all of the above settings it can be shown that relying on $\bm{X}_e^{2}$ to make predictions can lead to failures under distribution shifts (modeled by interventions). From \cite{arjovsky2019invariant}, we know that for the above family of models the ideal OOD predictor is $\big(\bm{\gamma},\bm{0}\big)$ as it performs well across many distribution shifts (modeled by interventions). Hence, the goal is to learn $\big(\bm{\gamma},\bm{0}\big)$. 
 
 \textbf{No confounders \& no anti-causal variables ($\bm{w}_{1}^{*}=\bm{w}_2^{*}$):} Consider the SEM in Assumption \ref{linear_model_ass}. For each environment $e\in \{1,2\}$, assume $\bm{\alpha}_e=0$ and $\bm{\Theta}_e=0$, i.e. no confounding and no anti-causal variables.  This setting captures the standard covariate shifts \cite{gretton2009covariate}, where it is assumed that $\mathbb{P}_e(Y_e|\bm{X}_e=\bm{x})$ is invariant across environments, here we assume $\mathbb{E}_{e}(Y_e|\bm{X}_e=\bm{x}) = \bm{\gamma}^{\mathsf{T}}\bm{x}$ is invariant across environments. The least squares optimal solution for each environment is $\bm{w}_e^{*}= \big(\bm{\gamma},\bm{0}\big)$, which implies that $\bm{w}_1^{*} = \bm{w}_2^{*}$. From Proposition \ref{prop1} we know that a NE exists (any two predictors adding to $\bm{w}_1^{*}$ form a NE). In this setting, different methods -- empirical risk minimization (ERM), IRM, IRM-games, and methods designed for covariate shifts such as sample reweighting --  should  perform well. 

\textbf{Confounders only ($\bm{w}_{1}^{*}\not=\bm{w}_2^{*}$):}
\label{cfd_only}
Consider the SEM in Assumption \ref{linear_model_ass}. For each environment $e\in \{1,2\}$, assume $\bm{\alpha}_e=\bm{0}$, $\bm{\Theta}_e\not=0$, i.e. confounders only setting. 
Define $\bm{\Sigma}_{e1}  = \mathbb{E}_{e}\Big[\bm{X}^1_{e}\bm{X}_{e}^{1,\mathsf{T}}\Big]$ and define the variance for the noise vector $\bm{\zeta}$ as  $\bm{\sigma}_{\bm{\zeta}_e}^2 = \mathbb{E}_e[\bm{\zeta}_e\odot \bm{\zeta}_e]$, where $\odot$ represents element-wise product between two vectors. 

\begin{assumption}
\textbf{Regularity condition for linear SEM in Assumption \ref{linear_model_ass}.} For each environment $e\in \{1,2\}$, $\bm{\Sigma}_{e1}$ is positive definite and each element of the vector $\bm{\sigma}_{\bm{\zeta}_e}^2$ is positive.
\label{ass2}
\end{assumption}

Assumption \ref{ass2} is equivalent to Assumption \ref{ass1} for SEM in Assumption \ref{linear_model_ass} (it ensures $\bm{\Sigma}_e$ is positive definite).
\begin{proposition}
\label{ls_cfd_prop}
If  Assumption \ref{linear_model_ass} holds with $\bm{\alpha}_e=\bm{0}$ for each $e\in \{1,2\}$, and Assumption \ref{ass2} holds, then the least squares optimal solution for environment $e$ is 
\begin{equation}
\bm{w}_e^{*} = (\bm{w}_{e}^{\mathsf{inv}}, \bm{w}_{e}^{\mathsf{var}})=\Big(\bm{\gamma}, \Big(\bm{\Theta}_e\bm{\Theta}_e^{\mathsf{T}} + \mathsf{diag}[\bm{\sigma}_{\bm{\zeta}_e}^{2}]\Big)^{-1} \bm{\Theta}_e\bm{\eta}_e\Big)
\label{sol_confouned}
\end{equation}
\end{proposition}

We divide $\bm{w}_e^{*}$ into two halves $\bm{w}_{e}^{\mathsf{inv}}= \bm{\gamma}$  and  $\bm{w}_{e}^{\mathsf{var}} = \Big(\bm{\Theta}_e\bm{\Theta}_e^{\mathsf{T}} + \mathsf{diag}[\bm{\sigma}_{\bm{\zeta}_e}^{2}]\Big)^{-1} \bm{\Theta}_e\bm{\eta}_e$. Observe that the first half $\bm{w}_{e}^{\mathsf{inv}}$ is invariant, i.e., it does not depend on the  environment, while  $\bm{w}_{e}^{\mathsf{var}}$ may vary as it depends on the parameters specific to the environment e.g., $\bm{\Theta}_e, \bm{\eta}_e$. In general, $\bm{w}_{1}^{\mathsf{var}}\not=\bm{w}_{2}^{\mathsf{var}}$ (e.g., $s=q$, $\bm{\Theta}_e$ is identity $\bm{I}_q$, $\bm{\sigma}_{\bm{\zeta}_e}^2$ is one $\bm{1}_q$,  $\bm{\eta}_1\not=\bm{\eta}_2$) and as a result  $\bm{w}_1^{*} \not= \bm{w}_2^{*}$. In such a case, from Proposition \ref{prop1}, we know that NE does not exist. 
ERM  and other techniques such as domain adaptation \cite{ajakan2014domain,ben2007analysis,glorot2011domain,ganin2016domain}, robust optimization \cite{mohri2019agnostic,hoffman2018algorithms,lee2018minimax,duchi2016statistics}, would tend to learn a model which tends to exploit information from the spuriously correlated $\bm{X}_e^{2}$ thus placing a non-zero weight on the second half corresponding to the features  $\bm{X}_e^{2}$ and not recovering  $(\bm{\gamma},\bm{0})$.  

IRM based methods are designed to tackle these problems.  These works try to learn representations that filter out causal features, $\bm{X}_e^{1}$, with invariant coefficients, $\bm{w}_{e}^{\mathsf{inv}}$, from spurious features, $\bm{X}_e^{2}$, with variant coefficients $\bm{w}_{e}^{\mathsf{var}}$  and learn a predictor on top resulting in the invariant predictor $ (\bm{\gamma},\bm{0})$. However, the current algorithms that search for these representations in IRM and IRM-games are based on gradient descent over non-convex losses and non-trivial best response dynamics respectively, both of which are  not guaranteed to  converge to the ideal OOD predictor $(\bm{\gamma},\bm{0})$.  We formally state the assumption underlying these methods, which we also use later.
\begin{assumption}\label{sp_var}
\textbf{Spurious features have varying coefficents across environments.} $\bm{w}_{1}^{\mathsf{var}}\not=\bm{w}_{2}^{\mathsf{var}}$ 
\end{assumption}

\subsection{Constrained Linear Regression Games}
\label{secn:clrg}
In U-LRG, $\Gamma$,  the utility of environment $1$ ($2$) for is  $-R_{1}(\bm{w}_{1},\bm{w}_{2})$\big($-R_{2}(\bm{w}_{1},\bm{w}_{2})$\big).  For each environment $e\in \{1,2\}$, $-R_{e}$ is continous and concave in $\bm{w}_{e}$. For each $e$ in the game $\Gamma$, the set of actions it can take is in $\mathbb{R}^{n \times 1}$, which is not a compact set. If the set of actions for each environment were compact and convex, then we can use Theorem \ref{thm1} to guarantee that a NE always exists. Let us constraint the predictors to be in the set $\mathcal{W} = \big\{\bm{w}_e\;\big|\;\|\bm{w}_{e}\|_{\infty} \leq w^{\mathsf{sup}} \big\}$, where $\|\cdot\|_{\infty}$ is the $\ell_{\infty}$ norm and $0<w^{\mathsf{sup}}<\infty$. We define the constrained linear regression game (C-LRG) as $\Gamma_c = \big(\mathcal{E}_{tr}, \{-R^{e}\}_{e\in \mathcal{E}_{tr}}, \mathcal{W}\big)$. $\mathcal{W}$ is a closed and bounded subset in the Euclidean space, which implies it is also a compact set. $\mathcal{W}$ is also a convex set as $\ell_{\infty}$ norm is convex. From Definition \ref{assm1:Existence}, the game $\Gamma_c$ is concave.   
\begin{proposition}
    A pure strategy Nash equilibrium always exists for  C-LRG, $\Gamma_c$.
\end{proposition}

The above proposition follows from Theorem \ref{thm1}.  Unlike the game $\Gamma$, a NE always exists for the game $\Gamma_c$. Let $\bm{w}_{1}^{\dagger}, \bm{w}_2^{\dagger}$ be a NE of $\Gamma_c$ and let $\bar{\bm{w}}^{\dagger}$ be the corresponding ensemble predictor, i.e. $\bar{\bm{w}}^{\dagger} = \bm{w}_1^{\dagger} + \bm{w}_2^{\dagger}$. In the next theorem, we analyze the properties of $\bm{w}_{1}^{\dagger}, \bm{w}_2^{\dagger}$ but before that we state some assumptions. 
\begin{assumption}
\textbf{Realizability.} For each $e\in \{1,2\}$ the least squares optimal solution $\bm{w}_e^{*}\in \mathcal{W}$.
\label{realize_ass}
\end{assumption}

We write the feature vector in environment $e$ as $\bm{X}_{e} = (X_{e1},\dots, X_{en})$ and the least squares optimal solution in environment $e$ as $\bm{w}_e^{*} = (w_{e1},\dots, w_{en})$. Divide the features indexed $\{1, \dots, n\}$ into two sets  $\mathcal{U}$ and $\mathcal{V}$.  $\mathcal{U}$ is defined as:   $i \in \mathcal{U}$ if and only if the  weight associated with $i^{th}$ component in the least squares solution is equal in the two environments, i.e., $w_{1i}^{*}=w_{2i}^{*}$. $\mathcal{V}$ is defined as: $i \in \mathcal{V}$ if and only if the  weight associated with $i^{th}$ component in the least squares solution is not equal in the two environments, i.e., $w_{1i}^{*}\not=w_{2i}^{*}$. For an example of these sets, consider the least squares solution to the confounded only SEM in equation \eqref{sol_confouned} under Assumption \ref{sp_var} ,   $\bm{w}_{1}^{\mathsf{inv}}=\bm{w}_{2}^{\mathsf{inv}}=\bm{\gamma}\;\implies \mathcal{U}=\{1,\dots, p\}$, and  $\bm{w}_{1}^{\mathsf{var}}\not=\bm{w}_{2}^{\mathsf{var}} \implies \mathcal{V} = \{p+1,\dots, p+q\}$.

\begin{assumption}
\textbf{Features with varying coefficients across environments are uncorrelated.} For each $i\in \mathcal{V}$ and for each $e\in \{1,2\}$, the corresponding feature $X_{ei}$ is uncorrelated with every other feature $j\in \{1,\dots,n\}\backslash\{i\}$, i.e., $\mathbb{E}[X_{ei}X_{ej}]=\mathbb{E}[X_{ei}]\mathbb{E}[X_{ej}]$.
\label{dis_ass}
\end{assumption}

The above assumption says that any feature component whose  least squares optimal solution coefficient varies across environments is not correlated with the rest of the features.  We use the above assumption to derive an analytical expression for the NE of $\Gamma_c$ next.  

For a vector $\bm{a}$, $|\bm{a}|$ represents the vector of absolute values of all the elements. Element-wise product of two vectors $\bm{a}$ and $\bm{b}$ is written as $\bm{a}\odot \bm{b}$. Define an indicator function $\bm{1}_{\bm{a} \geq \bm{b}}$; it carries  out an element-wise comparison of $\bm{a}$ and $\bm{b}$ and it outputs a vector of ones and zeros, where a one at component $i$ indicates that $i^{th}$ component of $\bm{a}$, $a_i$, is greater than or equal to the $i^{th}$ component of $\bm{b}$, $b_i$. Recall that the ensemble predictor constructed from NE is  $\bar{\bm{w}}^{\dagger} = \bar{\bm{w}}_1^{\dagger} +\bar{\bm{w}}_2^{\dagger}$.

\begin{theorem}
\label{nash_char:thm}

If Assumptions \ref{ass1}, \ref{realize_ass}, \ref{dis_ass} hold, then the ensemble predictor, $\bar{\bm{w}}^{\dagger}$, constructed from the Nash equilibrium, $(\bm{w}_1^{\dagger}, \bm{w}_2^{\dagger})$, of $\Gamma_c$   is equal to
\begin{equation}
    \Big(\bm{w}_1^{*} \odot \bm{1}_{|\bm{w}_2^{*}|\geq |\bm{w}_1^{*}|} + \bm{w}_2^{*} \odot \bm{1}_{|\bm{w}_1^{*}|>|\bm{w}_2^{*}|} \Big) \bm{1}_{ \bm{w}_1^{*}\odot \bm{w}_2^{*} \ge \bm{0}}
    \label{ne_exp}
\end{equation}

\end{theorem}

\textbf{Casewise analysis of NE in equation \eqref{ne_exp}}

$\bullet\; \bm{w}_1^{*}=\bm{w}_2^{*}$: Similar to Proposition \ref{prop1} $\big\{(\bm{w}_1^{\dagger}, \bm{w}_2^{\dagger}) \;| \; \bm{w}_1^{\dagger}\in \mathcal{W}, \bm{w}_2^{\dagger}\in \mathcal{W}, \bm{w}_1^{\dagger}+\bm{w}_2^{\dagger}= \bm{w}_1^{*}\big\}$ is the set of  NE of C-LRG

$\bullet$  $\bm{w}_1^{*}\not=\bm{w}_2^{*}$: We analyze this case under two categories
\begin{itemize}
    \item[$\square$] \textbf{Opposite  sign coefficients:} If  the $i^{th}$ component of $\bm{w}_1^{*}$ and $\bm{w}_2^{*}$ have opposite signs, then the $i^{th}$ component of the  ensemble predictor, $\bar{\bm{w}}^{\dagger}$, constructed from the NE of $\Gamma_c$, is zero, i.e., $\bar{w}^{\dagger}_{i} = \big[\bm{1}_{ \bm{w}_1^{*}\odot \bm{w}_2^{*} \ge \bm{0} }\big]_i=0$. In this case, the  coefficient of the environments' predictors in the NE, $w_{1i}^{\dagger}$ and $w_{2i}^{\dagger}$, have exact opposite signs and both are at the boundary one at  $w^{\mathsf{sup}}$ and other at $-w^{\mathsf{sup}}$. This case shows that when the features have a large variation in their least squares coefficients across environments, they can be spurious (see Proposition \ref{prop4}) and the  ensemble predictor filters them by assigning a zero weight to them.
    
    \item[$\square$] \textbf{Same  sign coefficients:} If  the $i^{th}$ component of $\bm{w}_1^{*}$ and $\bm{w}_2^{*}$ have same signs, then the $i^{th}$ component of ensemble predictor, $\bar{\bm{w}}^{\dagger}$, constructed from the NE of $\Gamma_c$, is set to the least squares coefficient with a smaller absolute value, i.e., $\bar{w}^{\dagger}_{i} =w_{1i}^{*}$, where $|w_{1i}^{*}|\leq |w_{2i}^{*}|$.  Suppose $0<w_{1i}^{*}<w_{2i}^{*}$,  the  coefficient of the environments' predictors in the NE, $w_{1i}^{\dagger}$ and $w_{2i}^{\dagger}$ have opposite signs, i.e., $w_{2i}^{\dagger} = w^{\mathsf{sup}}$ and $w_{1i}^{\dagger} = w_1^{*}- w^{\mathsf{sup}}$.  This shows that ensemble predictor is conservative and selects the smaller least squares coefficient. This property is useful to identifying predictors that are robust (see Proposition \ref{prop4}). Lastly, only when the least square coefficients are the same, i.e., $w_{1i}^{*}=w_{2i}^{*}$,  the coefficient of the environments' predictors in the NE can be in the interior, i.e.,  $|w_{1i}^{\dagger}|< w^{\mathsf{sup}}$ and $|w_{2i}^{\dagger}|<w^{\mathsf{sup}}$.
\end{itemize}

\subsubsection{Nash Equilibria for Linear SEMs}
Suppose for each environment $e\in \{1,2\}$ the data is generated from SEM in Assumption \ref{linear_model_ass}. We study if the NE of C-LRG, $\Gamma_c$, achieves or gets close to the ideal OOD predictor $(\bm{\gamma},0)$. We compare the ensemble predictors $\bm{\bar{w}}^{\dagger}$ constructed from the NE of $\Gamma_c$ to the solutions of ERM (Theorem \ref{nash_char:thm} enables this comparison). In ERM, the data from both the environments is combined and the overall least squares loss is minimized. Define the probability that a point is from environment $e$ as $\pi_{e}$ ($\pi_2=1-\pi_1)$. The set of ERM solutions for all distributions, $\{\pi_1,\pi_2\}$, is $\mathcal{S}^{\mathsf{ERM}}$ given as
\begin{equation*}
\Big\{\bm{w} \; \big| \pi_1 \in [0,1], \bm{w}\in \argmin_{\bm{\tilde{w}}\in \mathbb{R}^{n\times 1}}  \sum_{e\in \{1,2\}}\pi_{e}\mathbb{E}_e\big[\big(Y_e-\bm{\tilde{w}}^{\mathsf{T}}\bm{X}_e\big)^2\big] \Big\}
\end{equation*}

\begin{proposition}
\label{prop4}
If Assumption \ref{linear_model_ass} holds with $\bm{\alpha}_e=\bm{0}$ and $\bm{\Theta}_e$ an  orthogonal matrix for each $e\in \{1,2\}$, and Assumptions \ref{ass2}, \ref{sp_var}, \ref{realize_ass} hold, then $\|\bar{\bm{w}}^{\dagger} - (\bm{\gamma},0)\| < \|\bm{w}^{\mathsf{ERM}} - (\bm{\gamma},0)\|$ holds for all $\bm{w}^{\mathsf{ERM}} \in \mathcal{S}^{\mathsf{ERM}}$. \footnote{Exception occurs over measure zero set over probabilities $\pi_1$. If least squares solution are strictly ordered, i.e., $\forall i \in \{1,\dots,n\},0<w_{1i}^{*}< w_{2i}^{*}$ and $\pi_1=1$, then $ \bm{w}^{\mathsf{ERM}} = \bar{\bm{w}}^{\dagger}= \bm{w}_1^{*}$. In general, $\bm{w}_{1}^{*},\bm{w}_2^{*}$ are not ordered and $\pi_1\in(0,1)$, thus C-LRG improves over ERM.} Moreover, if all the components of two vectors $\bm{\Theta}_1\bm{\eta}_1$ and $\bm{\Theta}_2\bm{\eta}_2$ have opposite signs, then $\bar{\bm{w}}^{\dagger} = (\bm{\gamma},0)$.
\end{proposition}
We use Theorem \ref{nash_char:thm} to arrive at the above (Assumptions in the above proposition imply that Assumptions \ref{ass1}, \ref{realize_ass}, \ref{dis_ass} hold thus allowing us to use Theorem \ref{nash_char:thm}). From the first part of the above we learn that for many confounder only models ($\bm{\alpha}_e=\bm{0}$, $\bm{\Theta}_e$ an  orthogonal matrix), the ensemble predictor constructed from the NE is closer to the ideal OOD solution than ERM. For the second part, set $\bm{\Theta}_e = \bm{I}_{q}$, where $\bm{I}_q$ is identity matrix.  Suppose the signs of all the components of $\bm{\eta}_1$ and $\bm{\eta}_2$ disagree. As a result, the signs of latter half of least squares solution $\bm{w}_{e}^{\mathsf{var}}$ (in equation \eqref{sol_confouned}) disagree. From Theorem \ref{nash_char:thm}, we know that if the signs of the coefficients in least squares  solution disagree, then the corresponding coefficient in the ensemble predictor is zero, which implies $\bar{\bm{w}}^{\dagger}=(\bm{\gamma},0)$.

\textbf{Remark.} In Proposition \ref{prop4}, besides the regularity conditions, the main assumption is $\bm{\Theta}_e$ is orthogonal. This assumption ensures that the the spurious features $\bm{X}_e^2$ are uncorrelated (Assumption \ref{dis_ass}). For confounder only models this  seems reasonable. However, in the models involving anti-causal variables, i.e., $\bm{\alpha}_e\not=0$, the spurious features can be correlated and one may wonder how does the ensemble predictor behave in such setups? In experiments, we show that ensemble predictors perform well in these settings as well. Extending the theory to  anti-causal models is a part of future work.

\textbf{Insights from Theorem \ref{nash_char:thm}, Proposition \ref{prop4}}

Suppose the data comes from the SEM in Assumption \ref{linear_model_ass}. For this SEM, \cite{arjovsky2019invariant} showed that if the number of environments grow linearly in the total number of features, then the solution to non-convex IRM optimization recovers the ideal OOD predictor. We showed that for many confounder only SEMs ($\bm{\alpha}_e=0$ and $\bm{\Theta}_e$ orthogonal) NE based ensemble predictor gets closer to the OOD predictor than ERM and sometimes recovers it exactly with just \emph{two environments}, while \emph{no such guarantees} exist for IRM.  Next, we show how to learn these NE based ensemble predictor.

\begin{algorithm}
\SetAlgoLined
  \textbf{Initialize:} $\tilde{\bm{w}}_1= \bm{0}$, $\tilde{\bm{w}}_2 = \bm{0},$
  $p=0$
  
 \While{$w^{\mathsf{diff}}_1> 0$ or $w^{\mathsf{diff}}_2>0$}{
$\;\;\tilde{\bm{w}}_1^{\mathsf{cur}} = \tilde{\bm{w}}_1 $, $\tilde{\bm{w}}_2^{\mathsf{cur}} = \tilde{\bm{w}}_2 $

  $\;\;\tilde{\bm{w}}_1 = \min_{\bm{w}_1\in \mathcal{W}}R_1(\bm{w}_1,\tilde{\bm{w}}_2)\;$
  
  $\;\;\tilde{\bm{w}}_2 = \min_{\bm{w}_2\in \mathcal{W}}R_2(\tilde{\bm{w}}_1,\bm{w}_2)\;$
  
     $\;\;w^{\mathsf{diff}}_1 = \|\tilde{\bm{w}}_1^{\mathsf{cur}} - \tilde{\bm{w}}_1 \|$,
  $w^{\mathsf{diff}}_2 = \|\tilde{\bm{w}}_2^{\mathsf{cur}} - \tilde{\bm{w}}_2 \|$
 }
 \KwOut{$\bar{\bm{w}}^{+} = \tilde{\bm{w}}_1+\tilde{\bm{w}}_2$}
 \caption{Best response based learning}
 \label{algorithm1}
\end{algorithm}
\subsection{Learning NE of C-LRG}

In this section, we show how we can use best response dynamics (BRD) \cite{fudenberg1998theory} to learn the NE. Each environment takes its turn and finds the best possible model given the choice made by the other environment.  This procedure (Algorithm \ref{algorithm1}) is allowed to run until the environments stop updating their models. In  the next theorem, we make the same set of Assumptions as in Theorem \ref{nash_char:thm} and show that Algorithm \ref{algorithm1} converges to the  NE derived in Theorem \ref{nash_char:thm}. 
 \begin{theorem}
\label{brd:thm1}
If Assumption \ref{ass1}, \ref{realize_ass}, \ref{dis_ass} hold, then the output of Algorithm \ref{algorithm1}, $\bar{\bm{w}}^{+}$, is
\begin{equation*}
    \Big(\bm{w}_1^{*} \odot \bm{1}_{|\bm{w}_2^{*}|\geq |\bm{w}_1^{*}|} + \bm{w}_2^{*} \odot \bm{1}_{|\bm{w}_1^{*}|>|\bm{w}_2^{*}|} \Big) \bm{1}_{ \bm{w}_1^{*}\odot \bm{w}_2^{*} \ge \bm{0}}
    \label{ne_exp1}
\end{equation*}
\end{theorem}
We now understand different aspects of BRD.

\textbf{Dynamics of BRD.} Consider the $i^{th}$ component of the predictors $\tilde{w}_{1i}$ and $\tilde{w}_{2i}$ from Algorithm \ref{algorithm1}. Suppose $w_{1i}^{*}>w_{2i}^{*}$ and $|w_{1i}^{*}|>|w_{2i}^{*}|$. The two environments push the ensemble predictor, $\tilde{w}_{1i}+ \tilde{w}_{2i}$, in opposite directions during their turns, with the first environment increasing its weight, $\tilde{w}_{1i}$, and the second environment decreasing its weight, $\tilde{w}_{2i}$. Eventually, the environment with a higher absolute value ($e=1$ since $|w_{1i}^{*}|>|w_{2i}^{*}|$) reaches the boundary ($\tilde{w}_{1i}=w^{\mathsf{sup}}$) and cannot move any further due to the constraint. The other environment ($e=2$) best responds. It either hits the other end of the boundary ($\tilde{w}_{2i}=-w^{\mathsf{sup}}$), in which case the weight of the ensemble for component $i$ is zero, or gets close to the other boundary while staying in the interior ($\tilde{w}_{2i}=w_{2i}^{*}-w^{\mathsf{sup}}$), in which case the weight of the ensemble for component $i$ is $w_{2i}^{*}$.

\textbf{BRD a sequence of convex minimizations.} In Algorithm \ref{algorithm1}, we assumed that at each time step each environment can do an exact minimization operation. The minimization for each environment is a simple least squares regression, which is a convex quadratic minimization problem. There can be several ways of solving it  -- gradient descent for $R_e$ and solving for gradient of $R_e$ equals zero directly, which is a linear system of equations. We provide a simple bound for the total number of convex minimizations (or turns for each environment) in Algorithm \ref{algorithm1} next. For each $i\in \mathcal{V}$ (defined in Section \ref{secn:clrg}), compute the distance between the least square coefficients in the two environments $|w_{1i}^{*}-w_{2i}^{*}|$ and find the least distance over the set $\mathcal{V}$ given as $\Delta_{\mathsf{min}} = \min_{i\in \mathcal{V}}|w_{1i}^{*}-w_{2i}^{*}|$ (following the definition of $\mathcal{V}$ this distance is positive). The bound on number of minimizations is $\frac{2w^{\mathsf{sup}}}{\Delta_{\mathsf{min}}}$.

\subsubsection{Learning NE of C-LRG: Linear SEMs}

Suppose the data is generated from SEM in Assumption \ref{linear_model_ass}. Next, we show the final result that the NE based predictor, which we proved in Proposition \ref{prop4} is closer to the OOD solution, is achieved by Algorithm \ref{algorithm1}. 
 
\begin{proposition}
\label{prop5}
If Assumption \ref{linear_model_ass} holds with $\bm{\alpha}_e=\bm{0}$ and $\bm{\Theta}_e$ an  orthogonal matrix for each $e\in \{1,2\}$, and Assumptions \ref{ass2}, \ref{sp_var}, \ref{realize_ass} hold, then the output of Algorithm \ref{algorithm1}, $\bar{\bm{w}}^{+}$ obeys  $\|\bar{\bm{w}}^{+} - (\bm{\gamma},0)\| < \|\bm{w}^{\mathsf{ERM}} - (\bm{\gamma},0)\|$ for all $\bm{w}^{\mathsf{ERM}} \in \mathcal{S}^{\mathsf{ERM}}$ except over a set of measure zero (see footnote 2).  Moreover, if all the components of vectors $\bm{\Theta}_1\bm{\eta}_1$ and $\bm{\Theta}_2\bm{\eta}_2$ have opposite signs, then $\bar{\bm{w}}^{+}  = (\bm{\gamma},0)$.
\end{proposition}
We use Theorem \ref{brd:thm1} to arive at the above result. We have shown through Theorem \ref{nash_char:thm}, Proposition \ref{prop4}, Theorem \ref{brd:thm1} and Proposition \ref{prop5} that the NE based ensemble predictor of $\Gamma_c$ has good OOD properties and it can be learned by solving a sequence of convex quadratic minimizations. 

\textbf{Extensions:} In the Appendix, we extend the Theorem \ref{brd:thm1} to other BRD commonly used in literature (alternating gradient descent). We discuss convergence to NE when Assumption \ref{dis_ass} (uncorrelatedness) may not hold. Finally, we also have a high level discussion on extending the entire setup to  non-linear models and multiple environments.

\begin{figure*}[htbp]
\centering
  \includegraphics[width=0.4\textwidth]{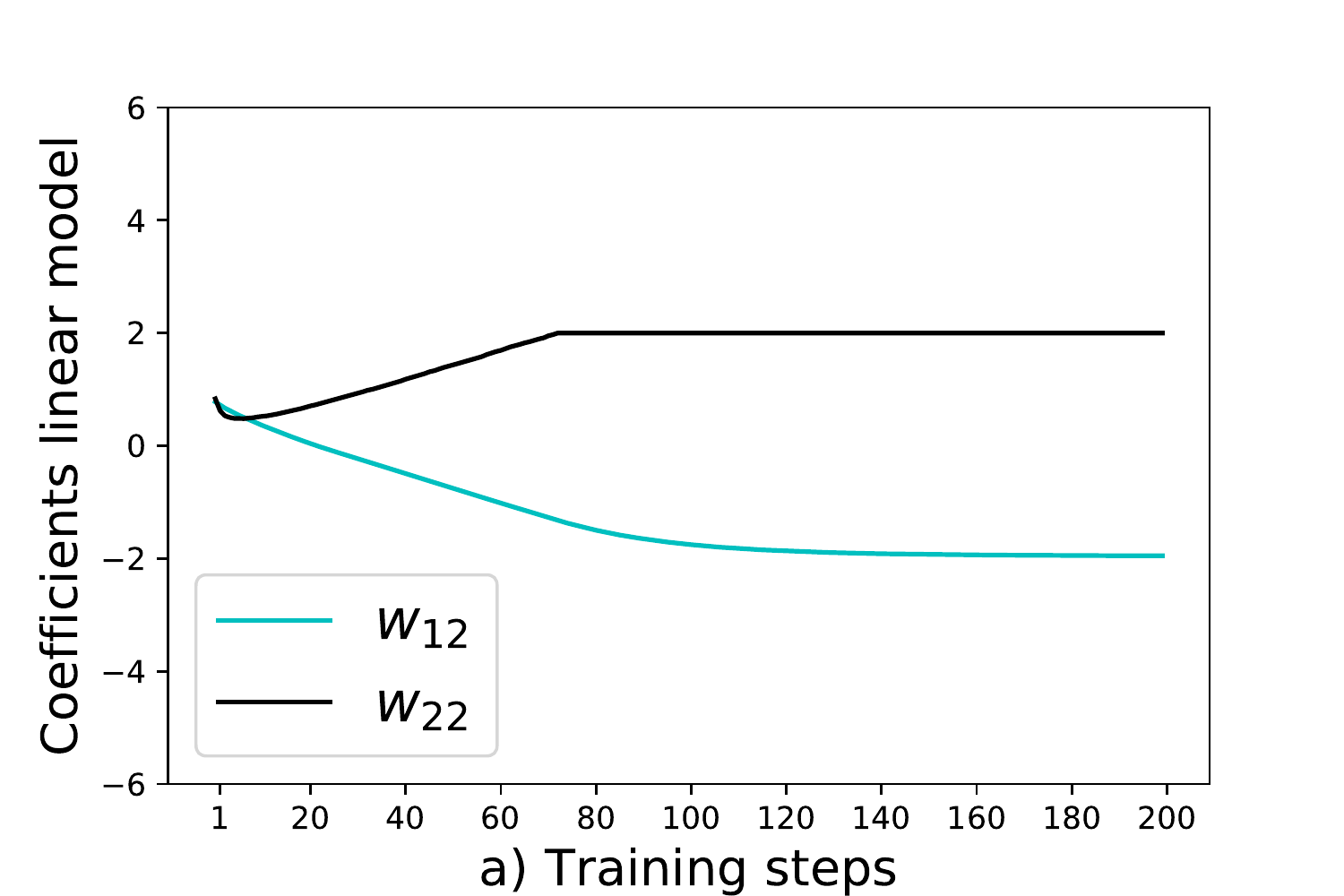}
  \includegraphics[width=0.4\textwidth]{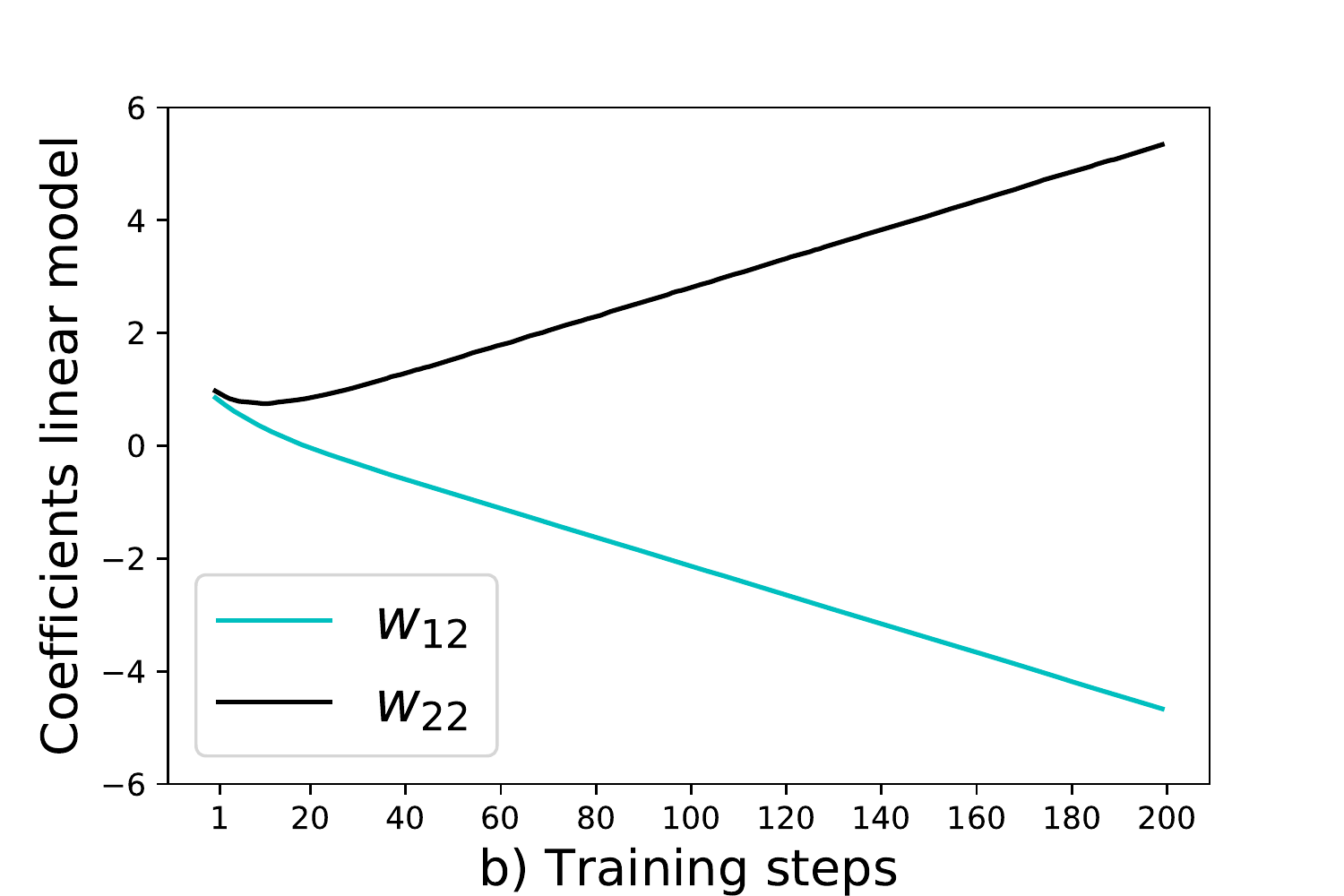}
  \includegraphics[width=0.4\textwidth]{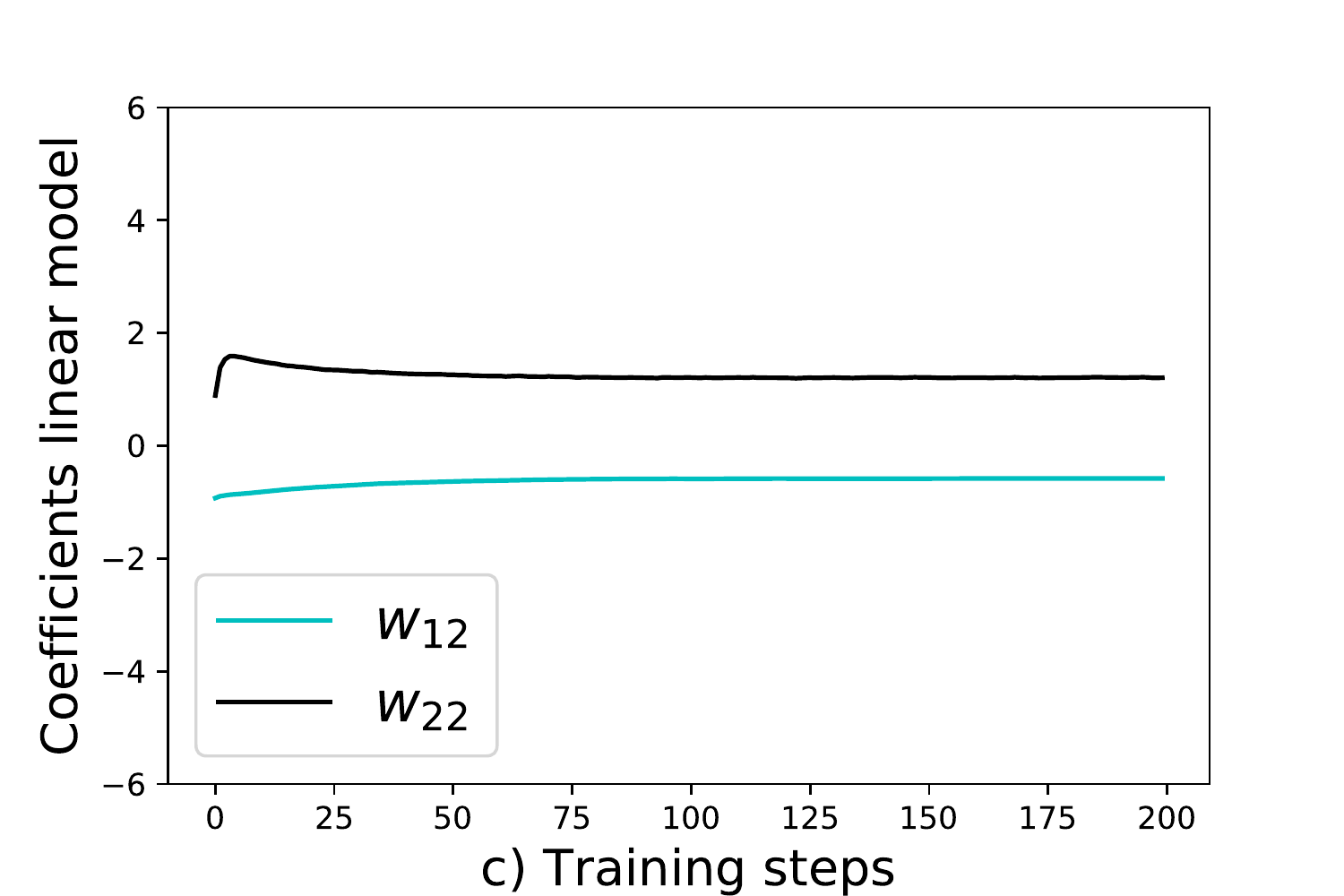}
  \includegraphics[width=0.4\textwidth]{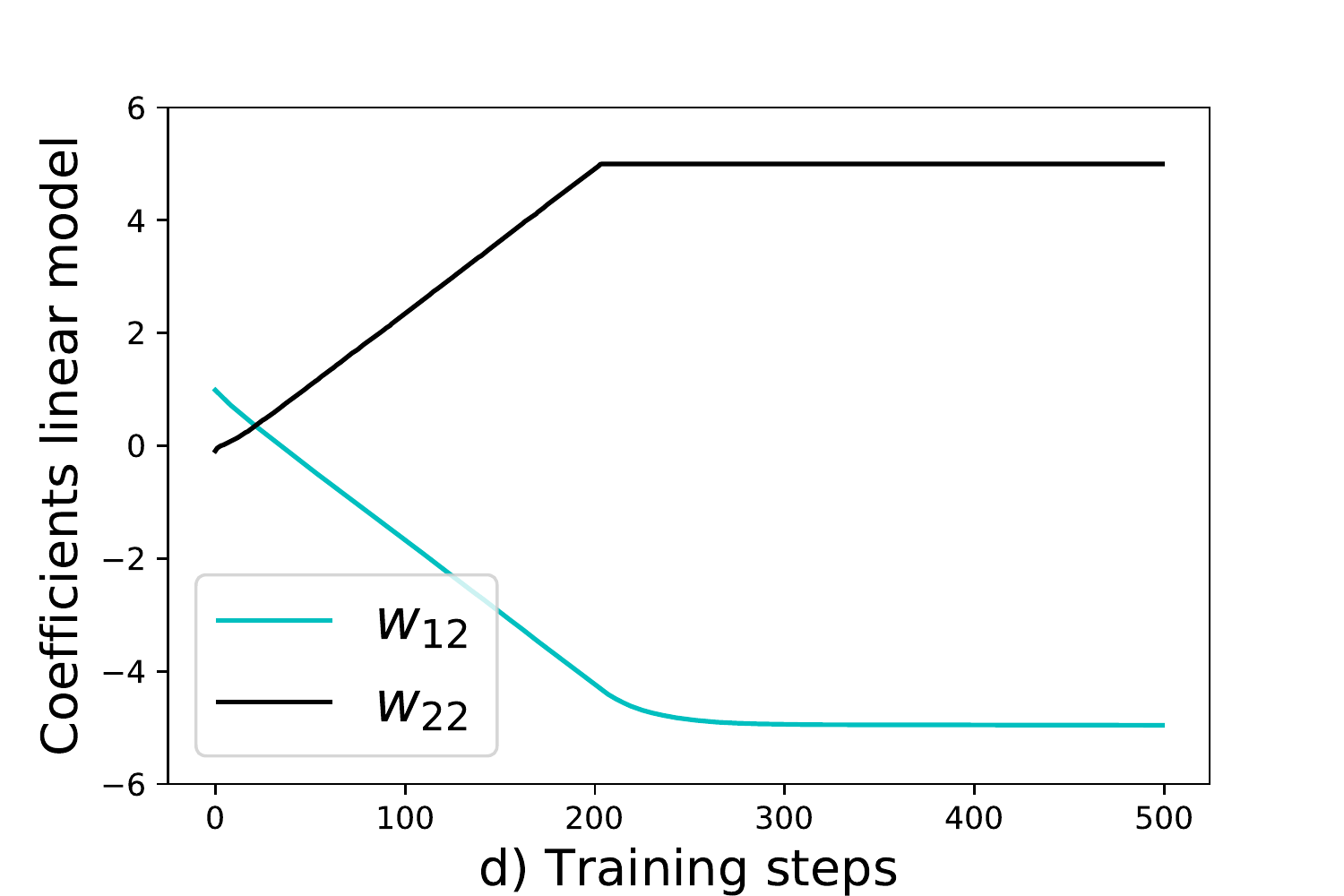}
\caption{\small{a) C-LRG ($w^{\mathsf{sup}}=2$), b) U-LRG, c) R$_{\infty}$-LRG , d) C-LRG ($w^{\mathsf{sup}}= 5$) }}
\label{fig1}
\end{figure*}

\section{Experiments}

In this section, we run the regression experiments described in \cite{arjovsky2019invariant}. We use the SEM in Assumption \ref{linear_model_ass} with following configurations.

  $\bullet$ $\bm{\gamma}$ is a vector of ones with $p$ dimensions, $\bm{1}_p$, which makes the ideal OOD model  $\big(\bm{1}_p, \bm{0}_{q}\big)$.  Each component of the confounder $\bm{H}_e$ is drawn i.i.d. from $\mathcal{N}(0,\sigma_{\bm{H}_e}^2)$. $\sigma_{\bm{H}_1}=0.2$, $\sigma_{\bm{H}_2}=2.0$. We consider two  configurations for $\bm{\Theta}_e$ and $\bm{\eta}_e$. i) $\bm{\Theta}_e=\bm{0}, \bm{\eta}_e=\bm{0}$, thus there is full observability (F) as there are no confounding effects, ii)   each component of $\bm{\Theta}_e$ and $\bm{\eta}_e$ is drawn i.i.d. from $\mathcal{N}(0,1)$ thus there is partial observability (P) as there are confounding effects.

$\bullet$ Each component of $\bm{\alpha}_e$ is drawn i.i.d from $ \mathcal{N}(0,1)$. $\varepsilon_e\sim \mathcal{N}(0, \sigma_{\varepsilon_e}^2)$ and each component of the vector $\bm{\zeta}_e$ is drawn from $\mathcal{N}(0,\sigma_{\bm{\zeta}_{e}}^2)$.  We  consider two settings for the noise variances -- Homoskedastic (HOM) $\sigma_{\varepsilon_1}=0.2$ and  $\sigma_{\varepsilon_2}=2.0$, $\sigma_{\bm{\zeta}_{1}}=\sigma_{\bm{\zeta}_{2}} = 1.0$  and Heteroskedastic (HET) $\sigma_{\bm{\zeta}_1}=0.2$ and  $\sigma_{\bm{\zeta}_2}=2.0$, $\sigma_{\varepsilon_{1}}=\sigma_{\varepsilon_{2}}=1.0$. 

\begin{table}
    \centering
     \renewcommand{\arraystretch}{1.25}
 \begin{tabular}{||c c c ||} 
 \hline
 \textbf{Method} & \textbf{Solution} & \textbf{Error} \\ [0.5ex] 
 \hline\hline
 Oracle            & $(1.0,0.0)$     & $0.0$ \\ \hline
 U-LRG  &$(0.34, 0.67)$ & $0.88$  \\ \hline
\textbf{C-LRG} $(w^{\mathsf{sup}} =  2)$ &  $(0.95, 0.05)$& $\bm{0.005}$  \\
 \hline
 \textbf{C-LRG} $(w^{\mathsf{sup}} = 5)$&  $(0.95, 0.04)$& $\bm{0.005}$  \\
 \hline
R$_\infty$-LRG &$(0.33, 0.65)$ & $0.87$  \\   \hline 
R$_2$-LRG  &$(0.33, 0.63)$ & $0.83$ \\  \hline 
ERM & $(0.34, 0.67)$ & $0.88$   \\ \hline
IRM & $(0.63, 0.44)$& $0.33$ \\ \hline
ICP & $(0.0, 0.0)$& $1.0$  \\ \hline
\end{tabular}
    \caption{\small{Comparing variants of LRG, IRM, ICP, and ERM.}}
    \label{tab:my_label}
\end{table}
\begin{figure*}[htbp]
    \centering
    \includegraphics[width=0.4\textwidth]{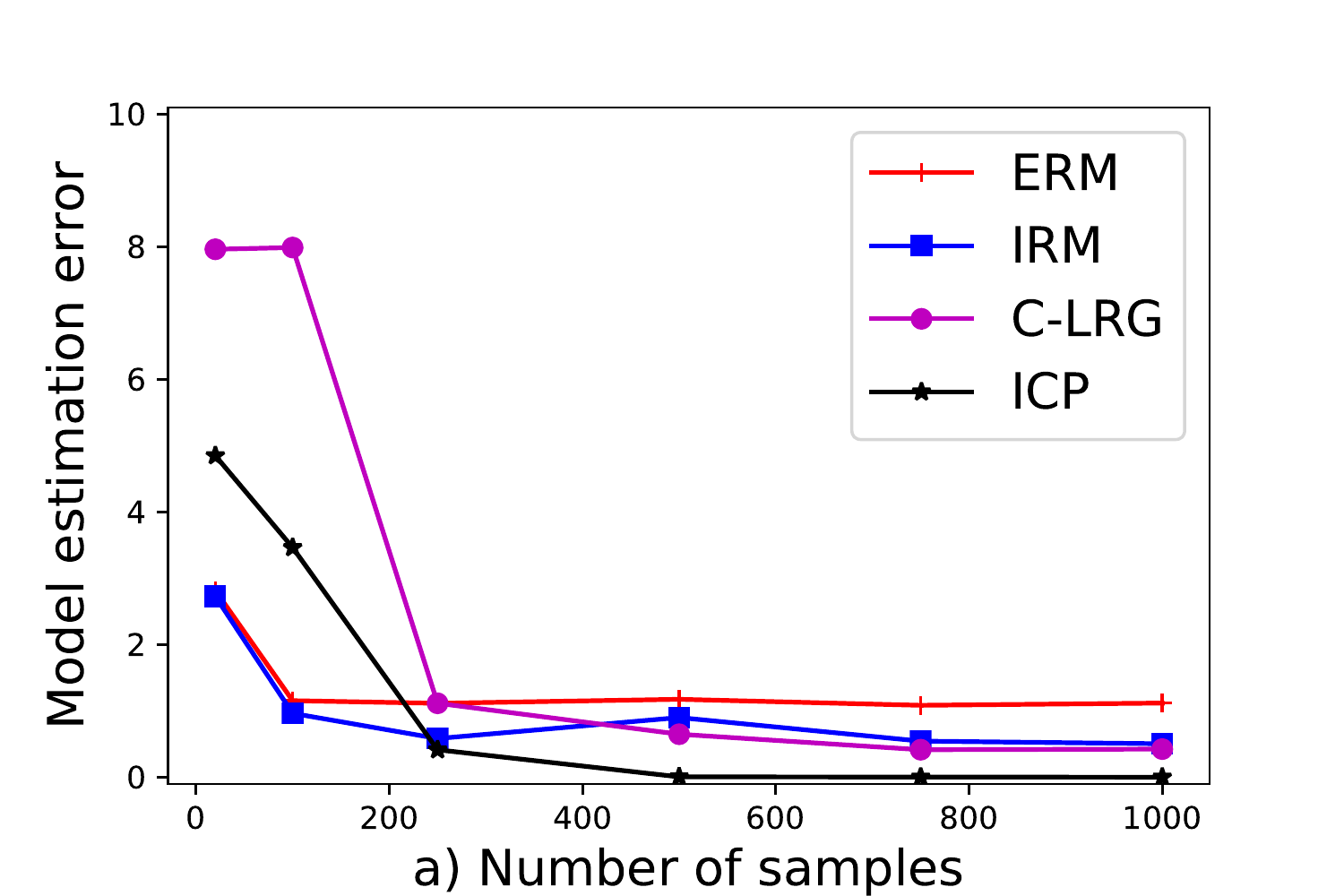}
    \includegraphics[width=0.4\textwidth]{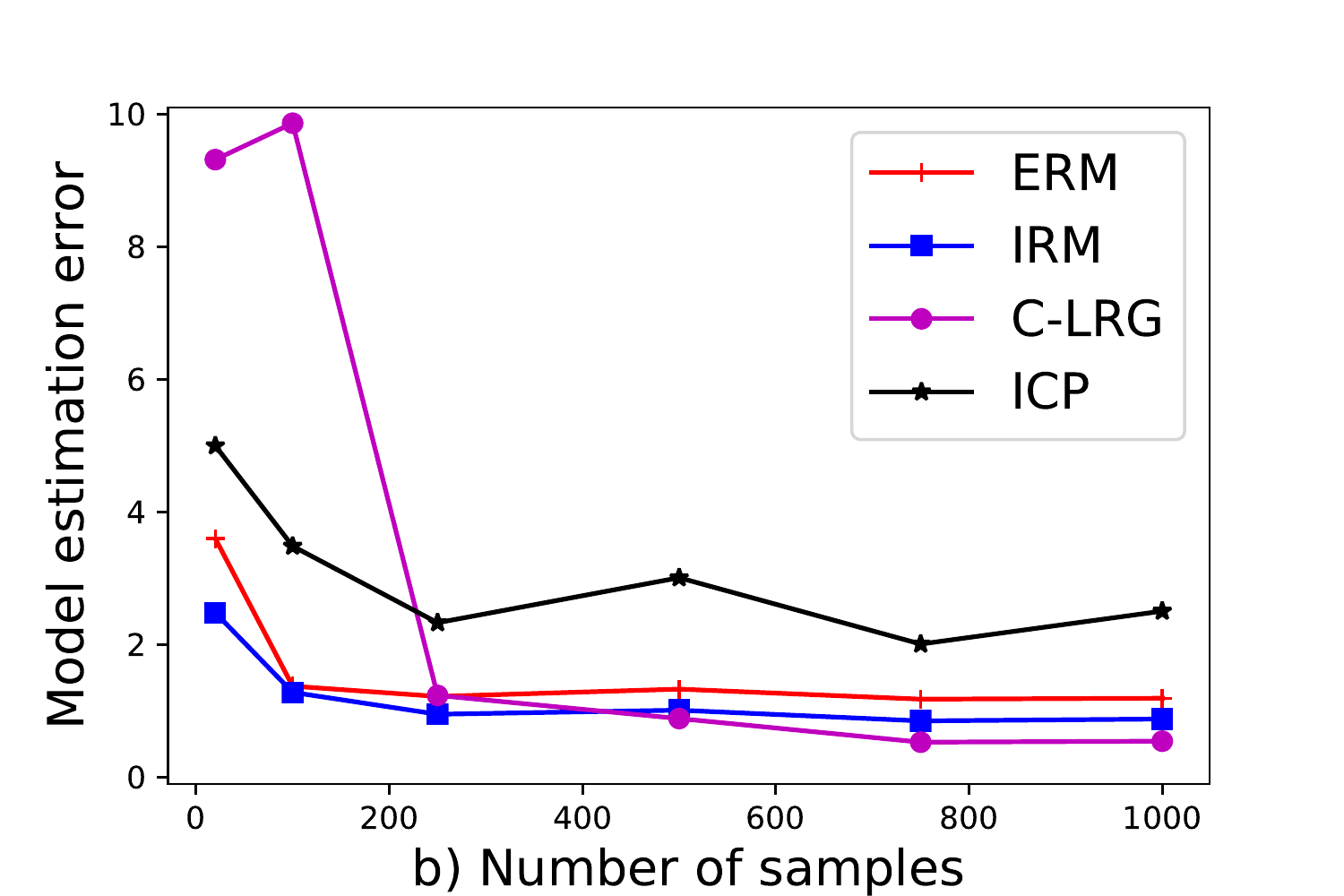}\\
    \includegraphics[width=0.4\textwidth]{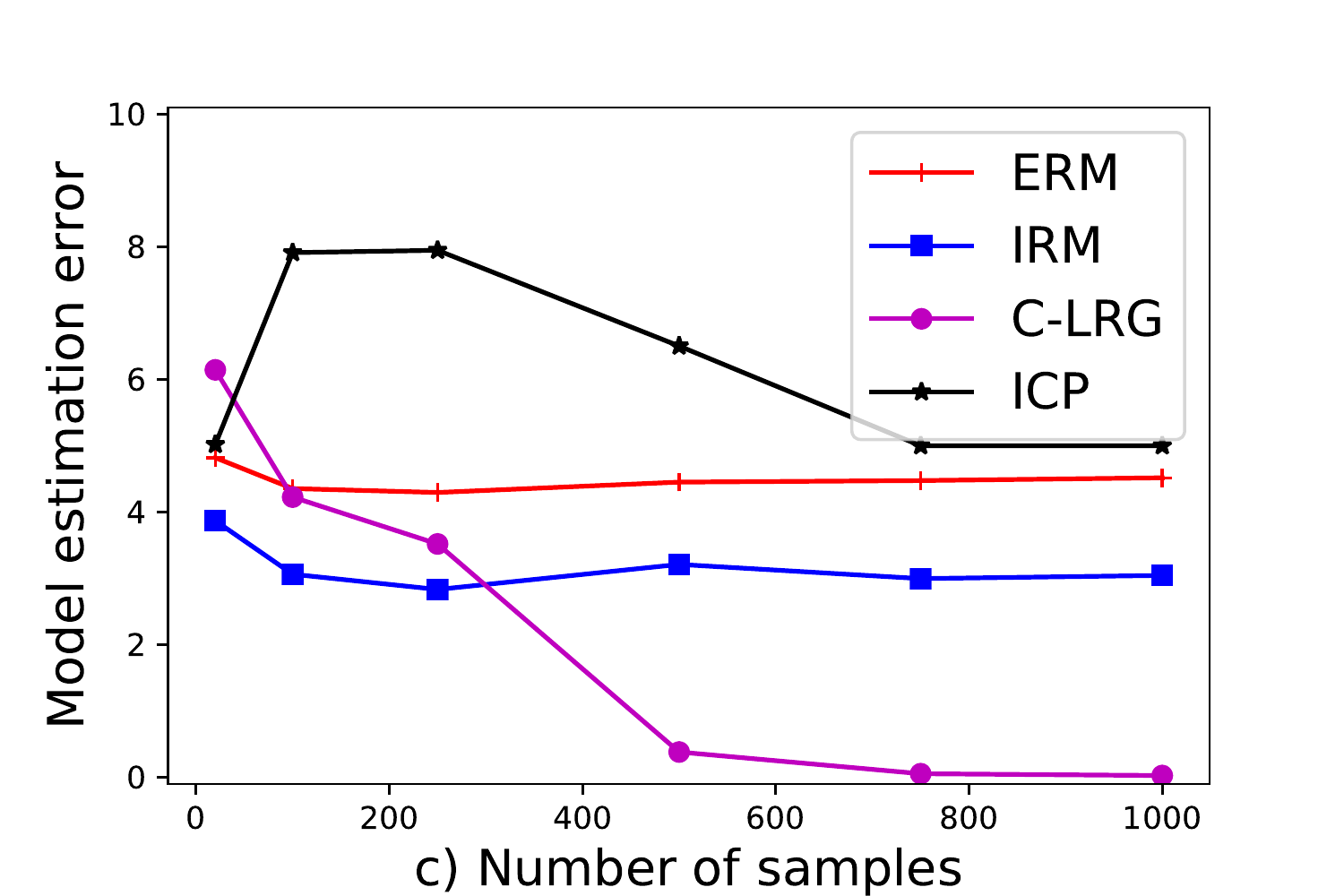}
    \includegraphics[width=0.4\textwidth]{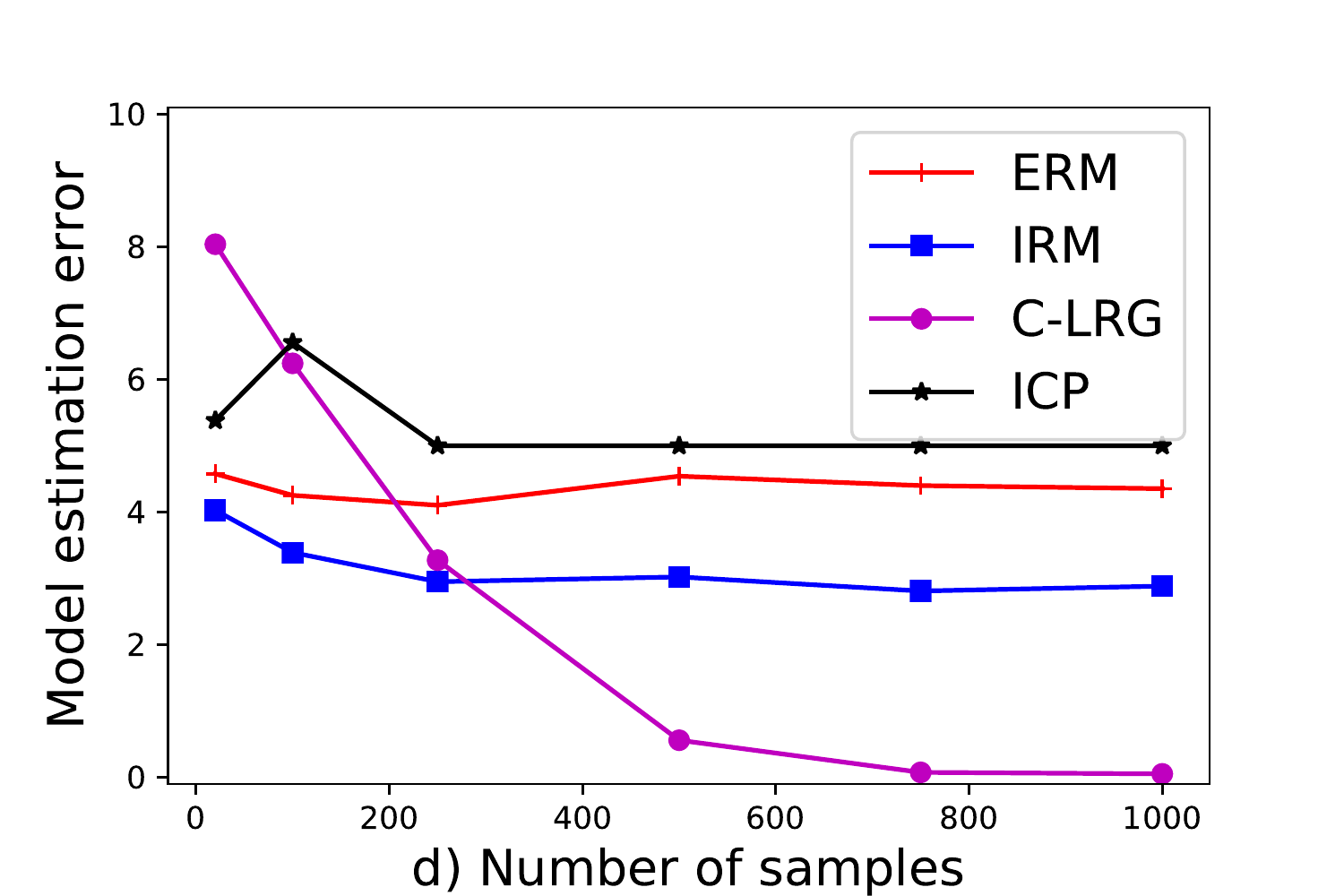}
     \caption{We compare across four settings: a) F-HET, b) P-HET, c) F-HOM and d) P-HOM.}
    \label{fig2}
\end{figure*}
From the above, we gather that there are four possible combination of settings in which comparisons will be carried out -- F-HOM, P-HOM, F-HET, P-HET.   We  use the following benchmarks in our comparison. IRM from \cite{arjovsky2019invariant}, ICP from  \cite{peters2015causal}, and standard ERM. Note in each of the cases we use a linear model. Our implementation for the data generation, IRM, ICP and ERM comes from \url{https://github.com/facebookresearch/InvariantRiskMinimization}. All other implementation details can be found in the Appendix. The source code is available at \url{https://github.com/IBM/IRM-games/tree/master/LRG_games}. The performance is measured in terms of the model estimation error, i.e., the square of the distance from the ideal model $(\bm{1}_{p},\bm{0}_q)$. 

Before we discuss a comparison in all these settings, we look at a two dimensional experiment where $p=q=1$ and the parameters are set to F-HOM. We carry out this comparison to illustrate several points. Firstly, we want to show why is $\ell_{\infty}$ constraint very important. Secondly, we want to show that the works when  $\bm{\alpha}_e$ is non-zero, i.e., $\bm{X}_e^2$ is anti-causal (in the theory we had assumed $\bm{\alpha}_e=0$). We compare with following variants of the linear regression game (LRG) i) no constraints, which is the game U-LRG (Section \ref{senc:ulrg}), ii) regularize  each $R_e$ with $\ell_{\infty}$ penalty (R$_\infty$-LRG), and iii) regularize each $R_e$ with $\ell_2$ penalty (R$_2$-LRG). In Table \ref{tab:my_label}, we show the estimated model against the respective method and the estimation error.  Observe that C-LRG was able to outperform other variants of LRG. Moreover, C-LRG performed better than the other existing methods as well.  $w_{12}$ ($w_{22}$) are the coefficients that model 1 (2) associates with feature 2, which is spuriously correlated. We plot the trajectories of the coefficients $w_{12}$ ($w_{22}$) of the models of each of the environments for the spurious features as the best response dynamics based training proceeds in Figure \ref{fig1}. Observe how the $\ell_{\infty}$ constrained models saturate on opposite ends of the boundary and as a result they cancel the spurious factors out. In contrast for other models, we do not see such an effect. Lastly, see  if we choose a larger bound $w^{\mathsf{sup}}=5$ the coefficients reach the boundary they just take more steps than  $w^{\mathsf{sup}}=2$.

Next, we move to a more elaborate comparison for the $10$ ($p=q=5$) dimensional setting from \cite{arjovsky2019invariant}. In Figure \ref{fig2}a, \ref{fig2}b, we show the  estimation error for F-HET and P-HET settings. In Figure \ref{fig2}c, \ref{fig2}d, we show the  estimation error for F-HOM and P-HOM settings.  Observe that in each of the settings how the C-LRG performs better than the rest or is close to the best approaches  when the number of samples is more than $400$.

\section{Conclusion}

In this work, we developed a new game-theoretic approach to learn OOD solutions for linear regressions. To the best of our knowledge, we have provided the first algorithms for which we can guarantee both convergence and better OOD behavior than standard empirical risk minimization. Experimentally too we see the promise of our approach as it is either competitive or outperforms the state-of-the-art by a margin. 

\section{Acknowledgement}
We would like to thank Dr. Kush R. Varshney for the valuable discussions in the initial stages of this work. 

\clearpage
\section{Appendix}

In this section, we provide the proofs to the propositions and theorems, and also provide other details on the experiments. We restate all the propositions and theorems for reader's convenience. In all our results, we use the following notation $\bm{a}$ is a vector, $a_i$ is the $i^{th}$ component of vector $\bm{a}$, $A$ is a scalar random variable, $\bm{A}$ is a vector random variable, $A_i$ is the $i^{th}$ component of the random variable $\bm{A}$, $\mathcal{A}$ is a set, bold capitalized Greek letters e.g., $\bm{\Sigma}$ are used for matrices. $\bm{I}_m$ is  a $m$ dimensional identity matrix and $\bm{1}_m$ is a $m$ dimensional vector of ones.  A bar over a vector $\bm{w}$, $\bar{\bm{w}}$, denotes the ensemble predictor (sum of predictor from the two environments). 

\subsection{Proposition 1}
We restate Proposition \ref{prop1} below. 
\begin{proposition}
    If Assumption \ref{ass1} holds and  if the least squares optimal solution in the two environments are 
    \begin{itemize}
    \item equal, i.e., $\bm{w}_{1}^{*}=\bm{w}_2^{*}$, then the set $\{(\bm{w}_1^{\dagger}, \bm{w}_2^{\dagger}) \;| \; \bm{w}_1^{\dagger}+\bm{w}_2^{\dagger}= \bm{w}_1^{*}\}$ describes all the pure strategy Nash equilibrium of U-LRG, $\Gamma$.
    \item  not equal, i.e., $\bm{w}_{1}^{*}\not=\bm{w}_2^{*}$, then U-LRG, $\Gamma$, has no pure strategy Nash equilibrium.
    \end{itemize}
    \label{prop1_append}
\end{proposition}

\begin{proof}
We start with latter part of the proposition.
Suppose there exists a pair $\bm{w}_1^{\dagger}, \bm{w}_2^{\dagger}$ which is a NE of U-LRG. 
Observe that $R_{e}(\bm{w}_1,\bm{w}_2)$ is jointly convex in $\bm{w}_1,\bm{w}_2$ ($R_{e}(\bm{w}_1,\bm{w}_2) = \mathbb{E}_{e}[(Y_{e}-\bm{w}_1^{\mathsf{T}}\bm{X}_{e}-\bm{w}_2^{\mathsf{T}}\bm{X}_{e})^2]$; loss inside the expectation is convex and expectation is a weighted sum over these losses). Let us compute the gradient of $R_{e}(\bm{w}_1,\bm{w}_2)$ w.r.t $\bm{w}_e$. 

\begin{equation}
\begin{split}
  &  \nabla_{\bm{w}_1} R_{1}(\bm{w}_1,\bm{w}_2)  = 2\bm{\Sigma}_1(\bm{w}_1+\bm{w}_2) - 2\bm{\rho}_1  \\ 
   &  \nabla_{\bm{w}_2} R_{2}(\bm{w}_1,\bm{w}_2)  = 2\bm{\Sigma}_2(\bm{w}_1+\bm{w}_2) - 2\bm{\rho}_2  
\end{split}
\label{proof1: eqn1}
\end{equation}

From the definition of pure strategy NE, it follows that $\bm{w}_1^{\dagger}$ ($\bm{w}_2^{\dagger}$) minimizes $R_1(\cdot, \bm{w}_2^{\dagger})$ ($R_2(\bm{w}_1^{\dagger},\cdot)$). From the convexity of $R_1(\cdot, \bm{w}_2^{\dagger})$ and $R_2(\bm{w}_1^{\dagger},\cdot)$ it follows that $\nabla_{\bm{w}_1|\bm{w}_1=\bm{w}_1^{\dagger}} R_{1}(\bm{w}_1,\bm{w}_2^{\dagger})=0$ and $\nabla_{\bm{w}_2|\bm{w}_2=\bm{w}_2^{\dagger}} R_{2}(\bm{w}_1^{\dagger},\bm{w}_2)=0$. Therefore, we have 
\begin{equation}
    \begin{split}
        & \bm{\Sigma}_1(\bm{w}_1^{\dagger}+\bm{w}_2^{\dagger}) -  \bm{\rho}_1 = 0 \implies   \bm{w}_1^{\dagger}+\bm{w}_2^{\dagger} = \bm{\Sigma}_{1}^{-1}\bm{\rho}_1= \bm{w}_1^{*}\;\\
        & \bm{\Sigma}_2(\bm{w}_1^{\dagger}+\bm{w}_2^{\dagger}) - \bm{\rho}_2 =0  \implies \bm{w}_1^{\dagger}+\bm{w}_2^{\dagger} =  \bm{\Sigma}_{2}^{-1}\bm{\rho}_2= \bm{w}_2^{*}\; 
    \end{split}
    \label{proof1:eqn2}
\end{equation}
In the above equation \eqref{proof1:eqn2}, we use Assumption \ref{ass1} and the optimal solution defined in Section \ref{senc:ulrg}, $\bm{w}_{e}^{*} = \bm{\Sigma}_{e}^{-1}\bm{\rho}_e$, for each $e\in \{1,2\}$.
From equations \eqref{proof1:eqn2}  it follows that $\bm{w}_1^{*}=\bm{w}_2^{*}$. Therefore, the existence of NE implies $\bm{w}_1^{*}=\bm{w}_2^{*}$ or in other words if $\bm{w}_1^{*}\not=\bm{w}_2^{*}$ implies NE does not exist. In the above we learned that $ \bm{w}_1^{*}=\bm{w}_2^{*}$ is a necessary condition for NE to exist. In the next part we show that this condition is sufficient as well. Suppose  $\bm{w}_1^{*}=\bm{w}_2^{*}=\bm{w}^{*}$. Define any point $\hat{\bm{w}}_1$ and another point $\hat{\bm{w}}_2 = \bm{w}^{*}-\hat{\bm{w}}_1$. Compute 
$\nabla_{\bm{w}_1|\bm{w}_1=\hat{\bm{w}}_1} R_{1}(\bm{w}_1,\hat{\bm{w}}_{2})$ and  $\nabla_{\bm{w}_2|\bm{w}_2=\hat{\bm{w}}_2} R_{2}(\hat{\bm{w}}_1,\bm{w}_2)$. Using the expression in equation \eqref{proof1: eqn1} we get 
\begin{equation}
\begin{split}
  &  \nabla_{\bm{w}_1| \bm{w}_1 = \hat{\bm{w}}_1} R_{1}(\hat{\bm{w}}_1,\hat{\bm{w}}_{2}) = 2\bm{\Sigma}_1(\hat{\bm{w}}_1 + \hat{\bm{w}}_2) - 2\bm{\rho}_1  \\
  & = 2\bm{\Sigma}_1\bm{w}^{*} - 2\bm{\rho}_1  = 0\; \text{(From the optimality of}\; \bm{w}^{*} \; \text{for} \; R_1 \text{)}\\
    & \nabla_{\bm{w}_2| \bm{w}_2 = \hat{\bm{w}}_2} R_{2}(\hat{\bm{w}}_1,\hat{\bm{w}}_2) = 2\bm{\Sigma}_2(\hat{\bm{w}}_1 + \hat{\bm{w}}_2) -2\bm{\rho}_2  \\ 
    & =2\bm{\Sigma}_2\bm{w}^{*} - 2\bm{\rho}_2  = 0\; \text{(From the optimality of}\; \bm{w}^{*} \; \text{for} \; R_2  \text{)}
    \end{split}
\end{equation}
From the convexity of $R_1$ and $R_2$ it follows that $\hat{\bm{w}}_1$, $\hat{\bm{w}}_2$ simultaneously minimize $R_1$ and $R_2$. Therefore, every such $\hat{\bm{w}}_1$ and $\hat{\bm{w}}_2$ that sum to $\bm{w}^{*}$ form a NE. This completes the proof. 
\end{proof}

\subsection{Proposition 2}
We restate Proposition \ref{ls_cfd_prop}  below. 
\begin{proposition}
If  Assumption \ref{linear_model_ass} holds with $\bm{\alpha}_e=\bm{0}$ for each $e\in \{1,2\}$, and Assumption \ref{ass2} holds, then the least squares optimal solution for environment $e$ is 
\begin{equation}
\bm{w}_e^{*} = (\bm{w}_{e}^{\mathsf{inv}}, \bm{w}_{e}^{\mathsf{var}})=\Big(\bm{\gamma}, \Big(\bm{\Theta}_e\bm{\Theta}_e^{\mathsf{T}} + \mathsf{diag}[\bm{\sigma}_{\bm{\zeta}_e}^{2}]\Big)^{-1} \bm{\Theta}_e\bm{\eta}_e\Big)
\label{sol_confouned_append}
\end{equation}
\end{proposition}
\begin{proof}
We derive the expression for the optimal predictor in the confounder only SEM in Assumption \ref{linear_model_ass}.
Recall that the general expression for the least squares optimal predictor (defined in Section \ref{senc:ulrg}) is 
\begin{equation}
\bm{w}_e^{*} = \bm{\Sigma}_e^{-1}\bm{\rho}_e
\label{eqn: proof_prop2_0}
\end{equation}

We use the SEM in Assumption \ref{linear_model_ass} to derive an expression for $\bm{\Sigma}_e$. First observe that from Assumption \ref{linear_model_ass}, we have that $\mathbb{E}_e[\bm{X}_e^{1}]=0$ and 

$$\mathbb{E}_{e}\big[Y_{e}\big] = \bm{\gamma}^{\mathsf{T}}\mathbb{E}_{e}\big[\bm{X}_{e}^{1}\big] + \bm{\eta}_{e}^{\mathsf{T}}\mathbb{E}_{e}\big[\bm{H}_{e}\big]+ \mathbb{E}_{e}\big[\varepsilon_{e}\big]=0$$
$$\mathbb{E}_{e}\big[\bm{X}^2_{e}\big] = \bm{\alpha}_e \mathbb{E}_{e}\big[Y_{e}\big] + \bm{\Theta}_e \mathbb{E}_{e}\big[\bm{H}_e\big]  + \mathbb{E}_{e}\big[\bm{\zeta}_e\big]=\bm{0}$$

Therefore
\begin{equation}
\mathbb{E}_e\big[\bm{X}_e^{1}\big]=0, \mathbb{E}_e\big[\bm{X}_e^{2}\big]=0
\label{eqn: proof_prop2_mean}
\end{equation}

We divide $\bm{\Sigma}_e$ into four smaller matrices 
$\bm{\Sigma}_{e1} = \bm{E}_{e}[\bm{X}_{e}^{1} \bm{X}_{e}^{1,\mathsf{T}}]$, $\bm{\Sigma}_{e2} = \bm{E}_{e}[\bm{X}_{e}^{2} \bm{X}_{e}^{2,\mathsf{T}}]$, $\bm{\Sigma}_{e12} = \bm{E}_{e}[\bm{X}_{e}^{1} \bm{X}_{e}^{2,\mathsf{T}}]$ and 
$\bm{\Sigma}_{e21} = \bm{E}_{e}[\bm{X}_{e}^{2} \bm{X}_{e}^{1,\mathsf{T}}]$.

From Assumption \ref{linear_model_ass}, we know $(\bm{H}_e,\bm{\zeta}_e) \perp \bm{X}_e^{1}$ and $\bm{X}_e^{2} \leftarrow \bm{\Theta}_e \bm{H}_e  + \bm{\zeta}_e$, which implies $\bm{X}_{e}^2\perp \bm{X}_e^{1}$. 

Therefore, from  $\bm{X}_{e}^2\perp \bm{X}_e^{1}$ and equation \eqref{eqn: proof_prop2_mean} it follows that 

\begin{equation}
\bm{\Sigma}_{e21} = \mathbb{E}_e\Big[\bm{X}_e^{2}\bm{X}_e^{1,\mathsf{T}}\Big]=\mathbb{E}_e\Big[\bm{X}_e^2\Big]\mathbb{E}_e\Big[\bm{X}_e^{1,\mathsf{T}}\Big]= \bm{0}_{q\times p}
\label{eqn: proof_prop2_1}
\end{equation}

\begin{equation}
\begin{split}
    & \bm{\Sigma}_{e2} = \mathbb{E}_{e}\Big[\bm{X}_e^2\bm{X}_e^{2,\mathsf{T}}\Big] 
    = \bm{\Theta}_e \mathbb{E}_{e}\Big[\bm{H}_e \bm{H}_e^{\mathsf{T}}\Big] \bm{\Theta}_e^{\mathsf{T}} + \bm{\Theta}_e \mathbb{E}_{e}\Big[\bm{H}_e\bm{\zeta}_e^{\mathsf{T}}\Big] + \mathbb{E}_{e}\Big[\bm{\zeta}_e\bm{H}_e^{\mathsf{T}}\Big]\bm{\Theta}_e^{\mathsf{T}} + \mathbb{E}\Big[\bm{\zeta}_e\bm{\zeta}_e^{\mathsf{T}}\Big]  \\&= \bm{\Theta}_e  \bm{\Theta}_e^{\mathsf{T}} + \mathsf{diag}[\bm{\sigma}_{\bm{\eta}_e}^2]
    \label{eqn: proof_prop2_2}
\end{split}
\end{equation}

In the above equation \eqref{eqn: proof_prop2_2}, we use  $\mathbb{E}_{e}\Big[\bm{H}_e \bm{H}_e^{\mathsf{T}}\Big] = \bm{I}_s$ and $\mathbb{E}_{e}\Big[\bm{H}_e\bm{\zeta}_e^{\mathsf{T}}\Big] = \bm{0}_{s\times q}$, which follow from Assumption \ref{linear_model_ass}. From Assumption \ref{ass2}, we know that $\bm{\sigma}_{\bm{\eta}_e}^2>\bm{0}$ and we use this observation in equation \eqref{eqn: proof_prop2_2} to deduce that $\bm{\Sigma}_{e2}$ is positive definite. 

From equation \eqref{eqn: proof_prop2_1} we can simplify $\bm{\Sigma}_e$ into a block diagonal matrix written as $\mathsf{diag}\Big[\bm{\Sigma}_{e1}, \bm{\Sigma}_{e2}\Big]$, where $\bm{\Sigma}_{e1} = \mathbb{E}_e\Big[\bm{X}_e^{1}, \bm{X}_e^{1,\mathsf{T}}\Big]$ and $\bm{\Sigma}_{e2} = \mathbb{E}_e\Big[\bm{X}_e^{2}, \bm{X}_e^{2,\mathsf{T}}\Big]$.  

 
 From Assumption \ref{ass2}, $\bm{\Sigma}_{e1}$ is positive definite and we showed above that $\bm{\Sigma}_{e2}$ is positive definite as well. Therefore, we can write the inverse of $\bm{\Sigma}_e$ as another block diagonal matrix written as 

\begin{equation}
\bm{\Sigma}_e^{-1} = \mathsf{diag}\Big[\bm{\Sigma}_{e1}^{-1}, \bm{\Sigma}_{e2}^{-1}\Big]
\label{eqn:sig_inv}
\end{equation}

Next let us simplify $\bm{\rho}_e = \Big[\mathbb{E}_{e}\Big[\bm{X}_{e}^{1}Y_{e}\Big], \mathbb{E}_{e}\Big[\bm{X}_e^{2}Y_{e}\Big]\Big]$.

\begin{equation}
    \mathbb{E}_{e}\Big[\bm{X}^1_{e}Y_{e}\Big] = \mathbb{E}_{e}\Big[\bm{X}^1_{e}\bm{\gamma}^{\mathsf{T}}\bm{X}^1_{e} + \bm{\eta}_{e}^{\mathsf{T}}\bm{H}_{e} + \varepsilon_{e}\Big] = \bm{\Sigma}_{e1}\bm{\gamma}
    \label{deriv:eqn1}
\end{equation}

\begin{equation}
\begin{split}
    & \mathbb{E}_{e}\Big[\bm{X}^2_{e}Y_{e}\Big] =\mathbb{E}_{e}\Big[\bm{X}^2_{e}(\bm{\gamma}^{\mathsf{T}}\bm{X}^1_{e} + \bm{\eta}_e^{\mathsf{T}}\bm{H}_{e} + \varepsilon_{e})] = \mathbb{E}_{e}[\bm{X}_{e}^{2}\bm{\eta}_{e}^{\mathsf{T}}\bm{H}^{e}\Big]  \\ 
& \mathbb{E}_{e}\Big[\bm{X}_{e}^{2}\bm{\eta}_{e}^{\mathsf{T}}\bm{H}^{e}\Big] = \mathbb{E}_{e}\Big[\bm{\Theta}_{e}\bm{H}_e\bm{\eta}_{e}^{\mathsf{T}}\bm{H}^{e}\Big] = \bm{\Theta}_{e} \mathbb{E}_e\Big[\bm{H}_e \bm{H}_e^{\mathsf{T}}\Big] \bm{\eta}_e   \\ 
& = \bm{\Theta}_e\bm{\eta}_e \; \text{\Big(Since}\; \mathbb{E}_{e}\Big[\bm{H}_e \bm{H}_e^{\mathsf{T}}\Big] = \bm{I}_s \; \text{\Big)}
\end{split}
\label{deriv:eqn2}
\end{equation}

Combining equations \eqref{eqn: proof_prop2_0}- \eqref{deriv:eqn2}, 

\begin{equation*}
    \bm{w}_e^{*} = \bm{\Sigma}_e^{-1}\bm{\rho}_e = \Big(\bm{\gamma}, \Big(\bm{\Theta}_e\bm{\Theta}_e^{\mathsf{T}} + \mathsf{diag}[\bm{\sigma}_{\bm{\zeta} e}^{2}]\Big)^{-1} \bm{\Theta}_e\bm{\eta}_e\Big)
\end{equation*}

This completes the derivation. 
\end{proof}
\subsection{Theorem \ref{nash_char:thm}} 
We first state a lemma needed for proving Theorem \ref{nash_char:thm}. 
\begin{lemma}
\label{lemma2}
Suppose Assumptions \ref{ass1} and \ref{realize_ass} hold. Consider the case when $\bm{w}_1^{*}\not=\bm{w}_2^{*}$. In this case, at least one of the predictors in the NE of C-LRG $\bm{w}_1^{\dagger}$ or $ \bm{w}_2^{\dagger}$ has to be on the boundary of the set, i.e. for at least one $e \in \{1,2\},$ $\|\bm{w}_e^{\dagger}\|_{\infty} = w^{\mathsf{sup}}$. Moreover, if $\|\bm{w}_1^{\dagger}\|_{\infty} < w^{\mathsf{sup}}$  ($\|\bm{w}_2^{\dagger}\|_{\infty} < w^{\mathsf{sup}}$) and $\|\bm{w}_{2}^{\dagger}\|_{\infty} = w^{\mathsf{sup}}$ ($\|\bm{w}_{1}^{\dagger}\|_{\infty} = w^{\mathsf{sup}}$) then the  ensemble predictor  is optimal for environment $e$, i.e., $\bar{\bm{w}}^{\dagger}=\bm{w}_{2}^{*}$ ($\bar{\bm{w}}^{\dagger}=\bm{w}_{1}^{*}$).

\end{lemma}

\begin{proof} We start with the first part of the above lemma.
In the first part, the only case that is excluded is when both the points forming the NE are in the interior, i.e., $\|\bm{w}_1^{\dagger}\|_{\infty} < w^{\mathsf{sup}}$ and $\|\bm{w}_2^{\dagger}\|_{\infty} < w^{\mathsf{sup}}$. Denote $\bm{w}_{-e}$ as the predictor used by the environment $q\in \{1,2\}\setminus \{e\}$.
We interhangeably use $R_{e}(\bm{w}_e, \bm{w}_{-e})$ and $ R_e(\bm{w}_1, \bm{w}_2)$. 
For environment $e$, from the definition of NE, it follows that $\bm{w}_e^{\dagger}$ satisfies $\bm{w}_e^{\dagger} \in \arg\min_{\bm{w}_e \in \mathcal{W}}R_{e}(\bm{w}_e, \bm{w}_{-e}^{\dagger})$. Note  i) $R_{e}(\bm{w}_e,\bm{w}_{-e}^{\dagger})$  is a convex function in $\bm{w}_e$, and ii) the set $\mathcal{W}$ has a non-empty relative interior (Since $w^{\mathsf{sup}}>0$). From these two conditions it follows that Slater's constraint qualification is satisfied, which implies strong duality holds \cite{boyd2004convex}. From strong duality, it follows that $\bm{w}_e^{\dagger}$ and $\lambda_e^{\dagger}$, where $\lambda_e^{\dagger}$ is the dual variable for the constraint $\|\bm{w}_e\|_{\infty} \leq w^{\mathsf{sup}}$, satisfy the KKT conditions given as follows 
\begin{equation}
    \begin{split}
        & \|\bm{w}_e^{\dagger}\| \leq w^{\mathsf{sup}} \\ 
        & \lambda_{e}^{\dagger} \geq 0 \\ 
        & \lambda_e^{\dagger}\big(\|\bm{w}_e^{\dagger}\|- w^{\mathsf{sup}}\big) =0  \\ 
        & 0\in \nabla_{\bm{w}_e^{\dagger}} R_{e}(\bm{w}_e^{\dagger},\bm{w}_{-e}^{\dagger}) + \lambda_e^{\dagger} \partial(\|\bm{w}_{e}^{\dagger}\|_{\infty}) 
    \end{split}
    \label{proof2: KKT_eqn}
\end{equation}
In the above $\partial(\|\bm{w}_{e}^{\dagger}\|_{\infty})$ represents the subdifferential of $\|\cdot\|_{\infty}$ at $\bm{w}_{e}^{\dagger}$.
If $\|\bm{w}_1^{\dagger}\|_{\infty}<w^{\mathsf{sup}}$ and $\|\bm{w}_2^{\dagger}\|_{\infty}<w^{\mathsf{sup}}$, then $\lambda_1^{\dagger}$ and $\lambda_{2}^{\dagger}$ are both zero. As a result, we have for $e\in \{1,2\}$,  $\nabla_{\bm{w}_e^{\dagger}} R_{e}(\bm{w}_e^{\dagger},\bm{w}_{-e}^{\dagger}) =0$. From the expression of the gradients in \eqref{proof1: eqn1} we have for each $e\in \{1,2\}$
\begin{equation}
\begin{split}
 &   \nabla_{\bm{w}_{e}^{\dagger}}R_{e}(\bm{w}_e^{\dagger}, \bm{w}_{-e}^{\dagger}) = 2\bm{\Sigma}_e(\bm{w}_{e}^{\dagger}+\bm{w}_{-e}^{\dagger}) - 2\bm{\rho}^{e} = 0 \\
&    \implies \bm{w}_{1}^{\dagger}+\bm{w}_{2}^{\dagger} = \bm{\Sigma}_e^{-1}\bm{\rho}^{e} = \bm{w}_{e}^{*}
\end{split}
 \label{proof2: eqn2}
\end{equation}
From equation \eqref{proof2: eqn2} it follows that $\bm{w}_1^{*}=\bm{w}_2^{*}$, which contradicts the assumption $\bm{w}_1^{*}\not=\bm{w}_2^{*}$. This completes the proof for the first part of the Lemma. 

Next, we move to the latter part of the proof, which states that if $\|\bm{w}_e^{\dagger}\|_{\infty} < w^{\mathsf{sup}}$  and $\|\bm{w}_{-e}^{\dagger}\|_{\infty} = w^{\mathsf{sup}}$, then the  ensemble predictor  is optimal for environment $e$, i.e., $\bar{\bm{w}}^{\dagger}=\bm{w}_{e}^{*}$. Since  $\|\bm{w}_e^{\dagger}\|_{\infty} < w^{\mathsf{sup}}$, from the KKT conditions above in \eqref{proof2: KKT_eqn} we have that $\lambda_e^{\dagger}=0$, which implies that  $\nabla_{\bm{w}_e^{\dagger}} R_{e}(\bm{w}_e^{\dagger},\bm{w}_{-e}^{\dagger}) =0$. Using the expression for gradient in equation \eqref{proof1: eqn1}, we have that $ \bm{w}_{e}^{\dagger}+\bm{w}_{-e}^{\dagger} = \bar{\bm{w}}^{\dagger} = \bm{w}_{e}^{*}$. This completes the proof. 
\end{proof}

We restate Theorem \ref{nash_char:thm} for reader's convenience.

\begin{theorem}
\label{nash_char:thm_append}

If Assumptions \ref{ass1}, \ref{realize_ass}, \ref{dis_ass} hold, then the ensemble predictor, $\bar{\bm{w}}^{\dagger}$, constructed from the Nash equilibrium, $(\bm{w}_1^{\dagger}, \bm{w}_2^{\dagger})$, of $\Gamma_c$   is equal to
\begin{equation}
    \Big(\bm{w}_1^{*} \odot \bm{1}_{|\bm{w}_2^{*}|\geq |\bm{w}_1^{*}|} + \bm{w}_2^{*} \odot \bm{1}_{|\bm{w}_1^{*}|>|\bm{w}_2^{*}|} \Big) \bm{1}_{ \bm{w}_1^{*}\odot \bm{w}_2^{*} \ge \bm{0}}
    \label{ne_exp_append}
\end{equation}

\end{theorem}

\begin{proof}

Recall in Section \ref{secn:clrg}, we divided the features $\{1,\dots,n\}$ into two sets $\mathcal{U}$ and $\mathcal{V}$. Without loss of generality assume that the first $k$ components in $\bm{X}_{e}$ belong to $\mathcal{U}$ and the next $n-k$ components to be in $\mathcal{V}$. Therefore, $\mathcal{U}=\{1,\dots,k\}$ and $\mathcal{V}=\{k+1,\dots,n\}$.  Define $\bm{X}_{e+}=(X_{e1},\dots, X_{ek})$ and $\bm{X}_{e-}= (X_{e(k+1)},\dots, X_{en})$. We divide the weights in $\bm{w}_{e} = (w_{e1},\dots, w_{en})$ into two parts where the weights associated with the first $k$ components, $\bm{X}_{e+},$ are $\bm{w}_{e+}=(w_{e1},\dots, w_{ek})$ and the weights associated with the next $n-k$ components, $\bm{X}_{e-}$, are $\bm{w}_{e-}=(w_{e(k+1)},\dots, w_{en})$.  Similarly, we divide the vector $\bm{\rho}_{e}$ defined in Section \ref{senc:ulrg} into $\bm{\rho}_{e+}$ and $\bm{\rho}_{e-}$.

Define $\bm{\Sigma}_{e+} = \mathbb{E}_e\big[\bm{X}_{e+}\bm{X}_{e+}^{\mathsf{T}}]$ and define $\bm{\Sigma}_{e-} = \mathbb{E}_e\big[\bm{X}_{e-}\bm{X}_{e-}^{\mathsf{T}}]$. As a consequence of the Assumption \ref{dis_ass}, we can simplify the expression for $\bm{\Sigma}_e$ as follows

\begin{equation}
    \bm{\Sigma}_e = \mathsf{diag}\Big[\bm{\Sigma}_{e+}, \bm{\Sigma}_{e-}\Big]
    \label{eqn:thm2_proof_0}
\end{equation}

For each $e\in \{1,2\}$, each feature component $i\in \{1,\dots,n\}$ has a mean zero $\mathbb{E}_{e}\big[X_{ei}\big]=0$. Therefore,  the variance in each feature component $i\in \{1,\dots,n\}$ is $\sigma_{ei}^{2} = \mathbb{E}_{e}\big[X_{ei}^2\big]$. We can further simplify $\bm{\Sigma}_{e-}$. Using Assumption \ref{dis_ass}, we  have that $\bm{\Sigma}_{e-}$ is a diagonal matrix, which we write as
\begin{equation}
    \bm{\Sigma}_{e-}= \mathsf{diag}[(\sigma_{em}^2)_{m=k+1}^{n}]\Big]
    \label{eqn:thm2_proof_1}
\end{equation}

We use equations \eqref{eqn:thm2_proof_0} and \eqref{eqn:thm2_proof_1} and the notation introduced above to simplify the risk as follows. 
\begin{equation}
\begin{split}
    &R_{e}(\bm{w}_1,\bm{w}_2) =  (\bm{w}_1+\bm{w}_2)^{\mathsf{T}}\bm{\Sigma}_{e}(\bm{w}_1+\bm{w}_2) - \rho_{e}^{\mathsf{T}}(\bm{w}_1+\bm{w}_2)  + \mathbb{E}_{e}\big[Y_{e}^2\big] = \\  
    &(\bm{w}_{1+}+\bm{w}_{2+})^{\mathsf{T}}\bm{\Sigma}_{e+}(\bm{w}_{1+}+\bm{w}_{2+}) - \bm{\rho}_{e+}^{\mathsf{T}}(\bm{w}_{1+}+\bm{w}_{2+}) + \\& \sum_{i=k+1}^{n} \big((w_{1i}+w_{2i})^2\sigma_{ei}^{2}-2(w_{1i}+w_{2i})\rho_{ei}\big) +  \mathbb{E}_{e}\big[Y_{e}^2\big]
    \end{split}
\end{equation}

Recall that $\bm{w}_{e}^{*} = \bm{\Sigma}_{e}^{-1}\bm{\rho}_e$ (defined in Section \ref{senc:ulrg}).  From the above equations \eqref{eqn:thm2_proof_0}, \eqref{eqn:thm2_proof_1} and Assumption \ref{ass1}, we get 
\begin{equation}
\begin{split}
    &\bm{w}_{e+}^{*} = \bm{\Sigma}_{e+}^{-1}\bm{\rho}_{e+} \\ 
    & \bm{w}_{e-}^{*} = \Big[\frac{\rho_{ei}}{\sigma_{ei}^2}\Big]_{ i \in  \{k+1,\dots, n\}}
\end{split}
    \label{eqn:thm2_proof_4}
\end{equation}
where $\bm{w}_{e+}^{*}$ is the vector of the first $k$ components in $\bm{w}_{e}^{*}$, $\bm{w}_{e-}^{*}$ are the next $n-k$ components in $\bm{w}_{e}^{*}$, $\rho_{ei}$, is the $i^{th}$ component of $\bm{\rho}_{e}$ and $\sigma_{ei}^2$ is the variance in $X_{ei}$. 

Recall that the first $k$ components comprise the set $\mathcal{U}$, which is defined as the set where the features of the least squares coefficients are the same across environments, i.e.,

\begin{equation}
    \bm{w}_{1+}^{*} = \bm{w}_{2+}^{*} 
    \label{eqn:thm2_proof_6_n}
\end{equation}

Define 
\begin{equation}
\begin{split}
& R_{e+}(\bm{w}_{1+},\bm{w}_{2+}) =   (\bm{w}_{1+}+\bm{w}_{2+})^{\mathsf{T}}\bm{\Sigma}_{e+}(\bm{w}_{1+}+\bm{w}_{2+}) - \bm{\rho}_{e+}^{\mathsf{T}}(\bm{w}_{1+}+\bm{w}_{2+}) +\mathbb{E}_{e}\big[Y_{e}^2\big]
\end{split}
\label{eqn:thm2_proof_3}
\end{equation}

For each $i\in \mathcal{V}=\{k+1,\dots,n\}$ define
\begin{equation}
R_{ei}(w_{1},w_2) =   \big((w_{1i}+w_{2i})^2-2(w_{1i}+w_{2i})w_{ei}^{*}\big)
\label{eqn:thm2_proof_2}
\end{equation}

We use the above equations \eqref{eqn:thm2_proof_3} and \eqref{eqn:thm2_proof_2} to simplify the risks as follows
\begin{equation}
\begin{split}
    & \min_{\bm{w}_e\in \mathcal{W}} R_{e}(\bm{w}_1,\bm{w}_2) =  \min_{\bm{w}_{e+}\in \mathcal{W}_{+}}R_{e+}(\bm{w}_{1+},\bm{w}_{2+})  +  \sum_{i=k+1}^{n}\sigma_{ei}^2\min_{|w_{ei}|\leq w^{\mathsf{sup}}} R_{ei}(w_{1i},w_{2i}) 
\end{split}
\label{eqn:thm2_proof_5}
\end{equation}

In the above $\mathcal{W}_{+} = \big\{\bm{w} \;|\; \bm{w}\in \mathbb{R}^{k},\| \bm{w}\|_{\infty} \leq w^{\mathsf{sup}} \big\}$. From the above expression in equation \eqref{eqn:thm2_proof_5}, we see that the the optimization for environment $e$ can be decomposed into separate smaller minimizations, which we analyze separately next.  $\ell_{\infty}$ norm constraints allows to make the problem in equation \eqref{eqn:thm2_proof_5} separable and for other norms such separability is not possible.  Henceforth, we will look at each smaller minimization as a separate game between the environments.

Let us consider the first minimization in equation \eqref{eqn:thm2_proof_5}
\begin{equation}
\min_{\bm{w}_{e+}\in \mathcal{W}_{1+}}R_{e+}(\bm{w}_{1+},\bm{w}_{2+})
\label{eqn:thm2_proof_6}
\end{equation}
Let us minimize the objective in equation \eqref{eqn:thm2_proof_6} without imposing the constraint that $\bm{w}_{1+}\in \mathcal{W}_{1+}$

\begin{equation}
\begin{split}
&    \bm{\Sigma}_{e+}(\bm{w}_{1+}+ \bm{w}_{2+}) = \bm{\rho}_{e+} \\  
& (\bm{w}_{1+}+ \bm{w}_{2+}) = \bm{\Sigma}_{e+}^{-1}\bm{\rho}_{e+} = \bm{w}_{e+}^{*} \; \text{(From equation} \;  \eqref{eqn:thm2_proof_4}\text{)}
\end{split}
\end{equation}

Therefore, we have that if $(\bm{w}_{1+}+ \bm{w}_{2+}) = \bm{w}_{e+}^{*}$, then environment $e$ achieves the minimum risk possible and cannot do any better. In fact, from equation \eqref{eqn:thm2_proof_4} since $\bm{w}_{1+}^{*} = \bm{w}_{2+}^{*}$, if $(\bm{w}_{1+}+ \bm{w}_{2+}) = \bm{w}_{1+}^{*}$, then both environments are at the minimum and cannot do any better. Therefore,  we know that all the elements in the set $\mathcal{C} = \{\bm{w}_{1+}, \bm{w}_{2+}\;|\; \bm{w}_{1+}\in \mathcal{W}^{+},\; \bm{w}_{1+}+ \bm{w}_{2+} = \bm{w}_{1+}^{*}\}$ form a NE of C-LRG. From Assumption \ref{realize_ass}, we know that this set $\mathcal{C}$ is non-empty.  Moreover, there are no points outside this set $\mathcal{C}$ which form a NE.  If $\bm{w}_{1+}+ \bm{w}_{2+} \not= \bm{w}_{1+}^{*}$, then the gradient will not be zero for either of the environments and both would prefer to move to a point where their gradients are zero. Hence, in every NE, $\bm{w}_{1+}+ \bm{w}_{2+} = \bm{w}_{1+}^{*}$. If we use the expression in equation \eqref{ne_exp} and compute first $k$ components it returns vector $\bm{w}_{1+}^{*}$ (we use the condition $\bm{w}_{1+}^{*} = \bm{w}_{2+}^{*}$ to simplify the expression in equation \eqref{ne_exp}). This shows that the expression in equation \eqref{ne_exp} correctly characterizes the NE for the first $k$ components that make up the set $\mathcal{U}$. We now move to the remaining $n-k$ components that make up the set $\mathcal{V}$.

 Consider a component $i\in \mathcal{V}= \{k+1, \dots, n\}$. Environment $e$ is interested in minimizing  $R_{ei}$ defined in equation \eqref{eqn:thm2_proof_3}. 
Let us consider the $i^{th}$ component of the expression in equation \eqref{ne_exp} in Theorem \ref{nash_char:thm}   and rewrite the expression in terms of scalars.

\begin{equation}
    \Big(w_{1i}^{*}  1_{|w_{2i}^{*}|\geq |w_{1i}^{*}|} + w_{2i}^{*} 1_{|w_{1i}^{*}|>|w_{2i}^{*}|} \Big) 1_{ w_{1i}^{*}w_{2i}^{*} \geq 0}
\end{equation}

We divide the analysis into two cases. 
In the first case, the signs of $w_{1i}^{*}$ and $w_{2i}^{*}$ disagree, which implies $1_{w_{1i}^{*}w_{2i}^{*} \geq 0}$ is zero. In the the second case, the signs of $w_{1i}^{*}$ and $w_{2i}^{*}$ agree, which implies $1_{w_{1i}^{*}w_{2i}^{*} \geq 0}$  is one. 
Let us start with the first case. Without loss of generality say $w_{1i}^{*}<0$ and $w_{2i}^{*}>0$.
Suppose $\bar{w}_{i}^{\dagger}>0$, where $\bar{w}_{i}^{\dagger}$ is the $i^{th}$ component of the NE based predictor
\begin{equation}
    \begin{split}
       & \bar{w}_{i}^{\dagger}>0 \implies w_{1i}^{\dagger} + w_{2i}^{\dagger} >0 \implies w_{1i}^{\dagger}>-w_{2i}^{\dagger} \; (w_{2i}^{\dagger}>-w_{1i}^{\dagger}) \\
       & \implies w_{1i}^{\dagger}>-w^{\mathsf{sup}} \; (w_{2i}^{\dagger}>-w^{\mathsf{sup}})
    \end{split}
    \label{ws_bd}
\end{equation}  
Observe that 
$$\frac{\partial R_{1i}(w_{1i},w_{2i}^{\dagger})}{\partial w_{1i}}\Big|_{w_{1i} = w_{1i}^{\dagger}} = 2(\bar{w}^{\dagger}_{i} -w_{1i}^{*}) >0$$ 

Since $w_{1i}^{\dagger}>-w^{\mathsf{sup}}$ \big(from equation \eqref{ws_bd}\big), $w_{1i}^{\dagger}$ can be decreased and improve the utility for environment $1$, which contradicts that  $w_{1i}^{\dagger}$ is NE. 
Suppose $\bar{w}_{i}^{\dagger}<0$, then from symmetry we can show that one of the environments will be able to increase the weight and improve its utility. 

Therefore, the only option that remains $\bar{w}_{i}^{\dagger}=0\implies w_{1i}^{\dagger}=-w_{2i}^{\dagger}$.

Observe that $\frac{\partial R_{1i}(w_{1i},w_{2i}^{\dagger})}{\partial w_{1i}}\big|_{w_{1i}=w_{1i}^{\dagger}} = 2( -w_{1i}^{*}) >0$ and if $w_{1i}^{\dagger}>-w^{\mathsf{sup}}$ environment $1$ will want to decrease $w_{1i}^{\dagger}$. 

Observe that $\frac{\partial R_{2i}(w_{1i}^{\dagger},w_{2i})}{\partial w_{2i}}\big|_{w_{2i}=w_{2i}^{\dagger}} = 2( -w_{2i}^{*}) <0$ and if $w_{2i}^{\dagger}<w^{\mathsf{sup}}$ environment $2$ will want to increase $w_{2i}^{\dagger}$.

Hence, the only solution left is for environment $1$ to be at $-w^{\mathsf{sup}}$ and environment $2$ to be at $w^{\mathsf{sup}}$. Environment $1$'s ($2$'s) risk decreases (increases) as it moves closer to its optimal point $w_{1i}^{*}$ ($w_{2i}^{*}$). When environment $2$  uses $w^{\mathsf{sup}}$, environment $1$'s best response is to use $-w^{\mathsf{sup}}$ as it brings the environment $1$ the closest it can get to $w_{1i}^{*}$. Therefore,  $(w^{\mathsf{sup}},-w^{\mathsf{sup}})$ is a NE.   This completes the first case, i.e., when the coefficients have opposite signs the coefficient of the NE based ensemble predictor for that component is $0$, which is what equation \eqref{ne_exp} states. 

Next, consider the case when the signs of $w_{1i}^{*}$ and $w_{2i}^{*}$ agree.  Let us consider the case when both have positive signs and the negative sign case will follow from symmetry. Suppose $0< w_{1i}^{*} <w_{2i}^{*}$. From Lemma \ref{lemma2}, we know that there are three scenarios possible.

In the first scenario, both $w_{1i}^{\dagger}$ and $w_{2i}^{\dagger}$ are on the same side of the boundary, say both are at $w^{\mathsf{sup}}$. 
\begin{equation}
\begin{split}
 &   \frac{\partial R_{1i}(w_{1i},w_{2i}^{\dagger})}{\partial w_{1i}}\big|_{w_{1i}=w_{1i}^{\dagger}}= 2(2w^{\mathsf{sup}}-w_{1i}^{*}) \\
& \frac{\partial R_{2i}(w_{1i}^{\dagger},w_{2i})}{\partial w_{2i}}\big|_{w_{2i}=w_{2i}^{\dagger}}= 2(2w^{\mathsf{sup}}-w_{2i}^{*})
\end{split}
\label{der_bd}
\end{equation}
From Assumption \ref{realize_ass}, $0<w_{1i}^{*}<w_{2i}^{*}\leq w^{\mathsf{sup}}$. Thus for both $e\in \{1,2\}$, from equation \eqref{der_bd} it follows that decreasing $w_{ei}^{\dagger}$ from the current state would improve the utility thus contradicting that they form a NE. 
The other possibility is that the two are on the other sides of the boundary, which makes  $\frac{\partial R_{1i}(w_{1i},w_{2i}^{\dagger})}{\partial w_{1i}}\big|_{w_{1i}=w_{1i}^{\dagger}}$ and $\frac{\partial R_{2i}(w_{1i}^{\dagger},w_{2i})}{\partial w_{2i}}\big|_{w_{2i}=w_{2i}^{\dagger}}$ negative (from Assumption \ref{realize_ass}); thus prompting each player on the negative side of the boundary to increase the weight and improve its utility, which contradicts the fact that they form a NE.

The other possibility arising out of Lemma \ref{lemma2} is $w_{1i}^{\dagger}=w^{\mathsf{sup}}$ and $w_{2i}^{\dagger}=w_{2i}^{*}-w^{\mathsf{sup}}$.
In this case, the $\frac{\partial R_{1i}(w_{1i},w_{2i}^{\dagger})}{\partial w_{1i}}\big|_{w_{1i}=w_{1i}^{\dagger}}$ is positive implying environment 1 can decrease and improve its utility. Thus this state is not a NE.

Therefore, the only remaining possibility is $w_{2i}^{\dagger}=w^{\mathsf{sup}}$ and 
$w_{1i}^{\dagger}=w_{1i}^{*}-w^{\mathsf{sup}}$. In this case, the $\frac{\partial R_{2i}(w_{1i}^{\dagger},w_{2i})}{\partial w_{2i}}\big|_{w_{2i}=w_{2i}^{\dagger}}$ is negative, environment $2$ cannot increase the weight further as it is already at the boundary (playing $w^{\mathsf{sup}}$ is a best response of environment $2$ brings it closest to the desired $w_{2i}^{*}$). Hence, this state is a NE and the ensemble predictor is at $w_{1i}^{*}$. 

If we suppose,  $w_{2i}^{*} <w_{1i}^{*}<0$. In this case, we can follow the exact same line of reasoning and arrive at the conclusion that the only NE is $\bar{w}^{\dagger}=w_{1i}^{*}$. 

We have analyzed all the possible cases when $|w_{1i}^{*}|<|w_{2i}^{*}|$ and both $w_{1i}^{*}$ and  $w_{2i}^{*}$ have the same sign.  This completes the proof for the first term in the expression in equation \eqref{ne_exp} 
$$w_{1i}^{*}  1_{|w_{2i}^{*}|\geq |w_{1i}^{*}|}$$ 
The second term  is same as the first term in equation \eqref{ne_exp} with the roles of environments swapped. 
Therefore, due to symmetry we do not need to work out the second term separately. This completes the analysis for all the cases in the equation \eqref{ne_exp} in Theorem \ref{nash_char:thm}.

\end{proof}

\subsection{Proposition \ref{prop4}}
We restate Proposition \ref{prop4} below. 
\begin{proposition}
\label{prop4_append}
If Assumption \ref{linear_model_ass} holds with $\bm{\alpha}_e=\bm{0}$ and $\bm{\Theta}_e$ an  orthogonal matrix for each $e\in \{1,2\}$, and Assumptions \ref{ass2}, \ref{sp_var}, \ref{realize_ass} hold, then $\|\bar{\bm{w}}^{\dagger} - (\bm{\gamma},0)\| < \|\bm{w}^{\mathsf{ERM}} - (\bm{\gamma},0)\|$ holds for all $\bm{w}^{\mathsf{ERM}} \in \mathcal{S}^{\mathsf{ERM}}$. \footnote{Exception occurs over measure zero set over probabilities $\pi_1$. If least squares solution are strictly ordered, i.e., $\forall i \in \{1,\dots,n\},0<w_{1i}^{*}< w_{2i}^{*}$ and $\pi_1=1$, then $ \bm{w}^{\mathsf{ERM}} = \bar{\bm{w}}^{\dagger}= \bm{w}_1^{*}$. In general, $\bm{w}_{1}^{*},\bm{w}_2^{*}$ are not ordered and $\pi_1\in(0,1)$, thus C-LRG improves over ERM.} Moreover, if all the components of two vectors $\bm{\Theta}_1\bm{\eta}_1$ and $\bm{\Theta}_2\bm{\eta}_2$ have opposite signs, then $\bar{\bm{w}}^{\dagger} = (\bm{\gamma},0)$.
\end{proposition}

\begin{proof}

We first show that the Assumptions made in the above proposition imply that the Assumptions needed for Theorem \ref{nash_char:thm} to be true hold. 

We show that Assumptions \ref{linear_model_ass}, \ref{ass2} $\implies$ Assumption \ref{ass1} holds. $\bm{X}_1^{e}$ is zero mean (from Assumption \ref{linear_model_ass}) and 
$$\mathbb{E}_{e}\big[Y_{e}\big] = \bm{\gamma}^{\mathsf{T}}\mathbb{E}_{e}\big[\bm{X}_{e}^{1}\big] + \bm{\eta}_{e}^{\mathsf{T}}\mathbb{E}_{e}\big[\bm{H}_{e}\big]+ \mathbb{E}_{e}\big[\varepsilon_{e}\big]=0$$
$$\mathbb{E}_{e}\big[\bm{X}^2_{e}\big] = \bm{\alpha}_e \mathbb{E}_{e}\big[Y_{e}\big] + \bm{\Theta}_e \mathbb{E}_{e}\big[\bm{H}_e\big]  + \mathbb{E}_{e}\big[\bm{\zeta}_e\big]=\bm{0}$$
Thus $\mathbb{E}_{e}\big[\bm{X}_e\big] = \mathbb{E}_{e}\big[(\bm{X}_e^{1}, \bm{X}_{e}^{2})\big]=\bm{0}$

In the proof of Proposition \ref{ls_cfd_prop}, we had shown that when the data is generated from SEM in Assumption \ref{linear_model_ass}

$\bm{\Sigma}_e = \mathsf{diag}\Big[\bm{\Sigma}_{e1}, \bm{\Sigma}_{e2}] = \mathsf{diag}\Big[\bm{\Sigma}_{e1}, \bm{\Theta}_e\bm{\Theta}_e^{\mathsf{T}} + \mathsf{diag}\big[\bm{\sigma}_{\bm{\zeta}_e}^{2}\big]\Big]$

Since $\bm{\Theta}$ is an orthogonal matrix we have 

\begin{equation}
\bm{\Sigma}_e = \mathsf{diag}\Big[\Sigma_{e1}, \mathsf{diag}[\bm{\sigma}_{\bm{\zeta}_e}^{2}+\bm{1}_{q}]\Big]
\label{eqn:sigma_e_exp}
\end{equation}

Both $\bm{\Sigma}_{e1}$ and $\mathsf{diag}[\bm{\sigma}_{\bm{\zeta}_e}^{2}+\bm{1}_q]$ are positive definite as a result $\bm{\Sigma}_e$ is also positive definite. Therefore, Assumption \ref{ass1} holds. 

The expression for the solution to the least squares optimal solution derived in equation \eqref{sol_confouned} has two parts $\bm{w}_{e}^{\mathsf{inv}}$ and $\bm{w}_{e}^{\mathsf{var}}$. Recall the definition of sets $\mathcal{U}$ and $\mathcal{V}$ from Section \ref{secn:clrg}. The first $p$ components corresponding to $\bm{w}_{e}^{\mathsf{inv}} \implies \{1,\dots,p\}\subseteq \mathcal{U}$. The next $q$ components corresponding to $\bm{w}_{e}^{\mathsf{var}}\implies \{p+1,\dots, p+q\} \supseteq \mathcal{V}$.  We showed above that $\bm{\Sigma}_{e}$ is a block diagonal matrix and the block corresponding to the feature components $\{p+1,\dots, p+q\}$ also equalling a diagonal matrix $\mathsf{diag}[\bm{\sigma}_{\bm{\zeta}_e}^{2}+\bm{1}_q]$. Therefore, we can see that each feature component in $\{p+1,\dots, p+q\}$ is uncorrelated with any other feature component. Therefore, Assumption \ref{dis_ass} also holds. Hence, all the Assumptions required for Theorem \ref{nash_char:thm} also hold. We write the expression for least squares optimal solution in this case (from equation \eqref{sol_confouned}) as $\bm{w}_e^{*} = (\bm{w}_{e}^{\mathsf{inv}}, \bm{w}_{e}^{\mathsf{var}})=\Big(\bm{\gamma}, \Big(\mathsf{diag}[\bm{\sigma}_{\bm{\zeta}_e}^{2}+\bm{1}_q]\Big)^{-1} \bm{\Theta}_e\bm{\eta}_e\Big)$. 
 
We divide the NE based ensemble predictor $\bar{\bm{w}}^{\dagger}$ into two halves:  $\bm{w}^{\dagger}_1$ is the vector of first $p$ coefficients of $\bar{\bm{w}}^{\dagger}$ and $\bm{w}^{\dagger}_2$ is the vector of next $q$ coefficients of $\bar{\bm{w}}^{\dagger}$. 

From Theorem \ref{nash_char:thm} it follows that
\begin{equation}
    \bm{w}^{\dagger}_1 = \bm{\gamma}
    \label{ne_ens_1}
\end{equation}
The next $q$ components in the set $\{1,\dots,q\}$ are computed as follows.  For $k\in \{1,\dots, q\}$, $(p+k)^{th}$ component of $\bm{w}_{e}^{*}$ is $\frac{\big[\bm{\Theta}_e\bm{\eta}_e\big]_{k}}{\sigma_{\bm{\zeta}_e,k}^2+1}$, 
where $\big[\bm{\Theta}_e\bm{\eta}_e\big]_{k}$ is the $k^{th}$ component of 
$\bm{\Theta}_e\bm{\eta}_e$ and $\sigma_{\bm{\zeta}_e,k}^2$ is the $k^{th}$ component of $\bm{\sigma}_{\bm{\zeta}_e}^2$. 

We first prove the latter part of the above proposition. If $\bm{\Theta}_1\bm{\eta}_1$ and $\bm{\Theta}_2\bm{\eta}_2$ have opposite signs, then for each $k\in \{1,\dots, q\}$, the sign of $(p+k)^{th}$ component of $\bm{w}_{1}^{*}$ and $\bm{w}_{2}^{*}$ are opposite. From Theorem \ref{nash_char:thm} it follows that $\bar{w}^{\dagger}_{p+k}=0$. This holds for all $k\in \{1, \dots, q\}$ and as a result we have 
$\bar{\bm{w}}^{\dagger} = \big(\bm{\gamma},\bm{0}\big)$.
Now we move to the former part of the Proposition, which compares the NE based ensemble predictor to ERM's solution.

ERM solves the following optimization problem
\begin{equation}
    \min_{\bm{w}\in \mathbb{R}^{n\times 1}} \pi_{1}\mathbb{E}_{1}\big[\big(Y_{1}-\bm{w}^{\mathsf{T}}\bm{X}_{1}\big)^2\big] + (1-\pi_{1})\mathbb{E}_{2} \big[\big(Y_{2}-\bm{w}^{\mathsf{T}}\bm{X}_{2}\big)^2\big]
\end{equation}

By putting the gradient of the above to zero, we get 
\begin{equation}
\begin{split}
&\big(    \pi_{1}\bm{\Sigma_1} + (1-\pi_{1})\bm{\Sigma}_2\big)\bm{w}^{\mathsf{ERM}} = \pi_{1}\bm{\rho}_{1}+ (1-\pi_{1})\bm{\rho}_2 \\ &
\bm{w}^{\mathsf{ERM}} = \big(    \pi_{1}\bm{\Sigma}_1 + (1-\pi_{1})\bm{\Sigma}_2\big)^{-1}(\pi_{1}\bm{\rho}_{1}+ (1-\pi_{1})\bm{\rho}_2) 
\end{split}
\label{erm_exp_0}
\end{equation}

Substituting the expression for $\bm{\Sigma}_e $ from equation \eqref{eqn:sigma_e_exp} into equation \eqref{erm_exp_0} we get $\bm{w}^{\mathsf{ERM}}$ equals

\begin{equation}
\begin{split}
(\bm{w}^{\mathsf{ERM}}_{1}, \bm{w}^{\mathsf{ERM}}_{2}) = \Big(\bm{\gamma}, (\pi_{1}\bm{\Theta}_{1}\bm{\eta}_1+ (1-\pi_{1})\bm{\Theta}_2\bm{\eta}_2)\odot \bm{\xi} \Big)
\end{split}
\label{erm_exp}
\end{equation}

where $\bm{\xi} = \bm{1}_q \oslash\Big(\pi_1\big(\bm{\sigma}_{\bm{\zeta}_1}^{2}+\bm{1}_q\big) + (1-\pi_1)\big(\bm{\sigma}_{\bm{\zeta}_2}^{2}+\bm{1}_q\big) \Big)$ and $\bm{a} \oslash\bm{b}$ is elementwise division of the two vectors $\bm{a}$ and $\bm{b}$, and  $\bm{w}^{\mathsf{ERM}}_{1} = \bm{\gamma}$ and $\bm{w}^{\mathsf{ERM}}_{2} = (\pi_{1}\bm{\Theta}_{1}\bm{\eta}_1+ (1-\pi_{1})\bm{\Theta}_2\bm{\eta}_2)\odot \bm{\xi} $. For $k\in \{1,\dots, q\}$, the $k^{th}$ component of $\bm{w}_{2}^{\mathsf{ERM}}$ is given as 
\begin{equation}
\frac{\pi_1\big[\bm{\Theta}_1\bm{\eta}_1\big]_{k} +\big(1-\pi_1)\big[\bm{\Theta}_2\bm{\eta}_2\big]_{k}}{\pi_1\big(\sigma_{\bm{\zeta}_1,k}^2+1\big) + (1-\pi_1)\big(\sigma_{\bm{\zeta}_2,k}^2+1\big)}
\end{equation}

Based on the ERM predictor (equation \eqref{erm_exp}) and NE-based ensemble predictor (equation \eqref{ne_ens_1}) correctly estimate the causal coefficients $\bm{\gamma}$, i.e., they match in the first $p$ coefficients.  We focus on the latter $q$ coefficients. The distance of ERM and NE based ensemble predictors are written as $\|\bm{w}^{\mathsf{ERM}} - (\bm{\gamma},\bm{0})\| = \|\bm{w}^{\mathsf{ERM}}_{2}\|$,  $\|\bar{\bm{w}}^{\dagger} - (\bm{\gamma},\bm{0})\| = \|\bar{\bm{w}}^{\dagger}_{2}\|$. Hence, we only need to compare the norm of $\bm{w}^{\mathsf{ERM}}_{2}$ and $\bar{\bm{w}}^{\dagger}_{2}$. 
From Assumption \ref{sp_var}, we know that $\bm{w}_1^{\mathsf{var}} \not= \bm{w}_2^{\mathsf{var}}$, thus the two differ in at least one component.  Consider a component $m$, where the two vectors $\bm{w}_1^{\mathsf{var}}$ and $ \bm{w}_2^{\mathsf{var}}$ do not match.
For simplicity, let us write $\Big[\bm{\Theta}_e\bm{\eta}_e\Big]_{m} = \vartheta_e$. Therefore,  $\frac{\vartheta_1}{\sigma_{\bm{\zeta}_1,m}^2+1} \not = \frac{\vartheta_2}{\sigma_{\bm{\zeta}_2,m}^2+1}$. 

There are two possiblities -- i) the signs of $\frac{\vartheta_1}{\sigma_{\bm{\zeta}_1,m}^2+1}$ and $\frac{\vartheta_2}{\sigma_{\bm{\zeta}_2,m}^2+1}$ do not match, and ii) the signs of $\frac{\vartheta_1}{\sigma_{\bm{\zeta}_1,m}^2+1}$ and $\frac{\vartheta_2}{\sigma_{\bm{\zeta}_2,m}^2+1}$  match.
In case i), the magnitude of the corresponding coefficient of NE based predictor is $0$. The magnitude for the ERM based predictor is given as 
\begin{equation}
    \Big|\frac{\pi_1\vartheta_1 +\big(1-\pi_1)\vartheta_2}{\pi_1\big(\sigma_{\bm{\zeta}_1,m}^2+1\big) + (1-\pi_1)\big(\sigma_{\bm{\zeta}_2,m}^2+1\big)}\Big|
\end{equation}

If $\pi_1\vartheta_1 +\big(1-\pi_1)\vartheta_2=0$ ($\pi_1 = \frac{\vartheta_2}{\vartheta_2-\vartheta_1}$), then the coefficient of ERM based solution has same magnitude as NE based predictor, which is equal to zero. Therefore, except for when $\pi_1 = \frac{\vartheta_2}{\vartheta_2-\vartheta_1}$, ERM is strictly worse than NE based ensemble predictor.

In case ii), the the signs of $\frac{\vartheta_1}{\sigma_{\bm{\zeta}_1,m}^2+1}$ and $\frac{\vartheta_2}{\sigma_{\bm{\zeta}_2,m}^2+1}$  match. Let us consider the case when both are positive.  Without loss of generality assume that $0\leq \frac{\vartheta_1}{\sigma_{\bm{\zeta}_1,m}^2+1}< \frac{\vartheta_2}{\sigma_{\bm{\zeta}_2,m}^2+1}$.  From Theorem 2, we know that the magnitude of the NE based predictor is equal to $\frac{\vartheta_1}{\sigma_{\bm{\zeta}_1,m}^2+1}$ and the magnitude of ERM based predictor is 
\begin{equation}
\frac{\pi_1\vartheta_1 +\big(1-\pi_1)\vartheta_2}{\pi_1\big(\sigma_{\bm{\zeta}_1,m}^2+1\big) + (1-\pi_1)\big(\sigma_{\bm{\zeta}_2,m}^2+1\big)}
\end{equation}
We take a difference of the magnitudes of the two and get

\begin{equation}
\begin{split}
    & \frac{\pi_1\vartheta_1 +\big(1-\pi_1)\vartheta_2}{\pi_1\big(\sigma_{\bm{\zeta}_1,m}^2+1\big) + (1-\pi_1)\big(\sigma_{\bm{\zeta}_2,m}^2+1\big)} - \frac{\vartheta_1}{\sigma_{\bm{\zeta}_1,m}^2+1} \\ 
    & \frac{\big(1-\pi_1)\Big(\vartheta_2\big(\sigma_{\bm{\zeta}_1,m}^2+1\big) -\vartheta_1 \big(\sigma_{\bm{\zeta}_2,m}^2+1\big)\Big)}{\Big(\pi_1\big(\sigma_{\bm{\zeta}_1,m}^2+1\big) + (1-\pi_1)\big(\sigma_{\bm{\zeta}_2,m}^2+1\big)\Big)\big(\sigma_{\bm{\zeta}_1,m}^2+1\big)}
    \end{split}
    \label{diff_mag}
\end{equation}

Since $\frac{\vartheta_1}{\sigma_{\bm{\zeta}_1,m}^2+1}< \frac{\vartheta_2}{\sigma_{\bm{\zeta}_2,m}^2+1}$ it follows that if $\pi_1\in [0,1)$, then the above difference in equation \eqref{diff_mag} is positive. However, if $\pi_1=1$, then the difference is zero. Therefore, except for when $\pi_1 = 1$, ERM is strictly worse than NE based ensemble predictor.

Lastly, the analysis for the case when both coefficients are negative also follows on exactly the above lines. Without loss of generality consider the case, $ \frac{\vartheta_2}{\sigma_{\bm{\zeta}_2,m}^2+1} <\frac{\vartheta_1}{\sigma_{\bm{\zeta}_1,m}^2+1}\leq 0$.  In this case, NE based predictor will take the value $\frac{\vartheta_1}{\sigma_{\bm{\zeta}_1,m}^2+1}$ (follows from Theorem \ref{nash_char:thm}) and its magnitude is $-\frac{\vartheta_1}{\sigma_{\bm{\zeta}_1,m}^2+1}$.  The magnitude of ERM based predictor is 
\begin{equation}
-\frac{\pi_1\vartheta_1 +\big(1-\pi_1)\vartheta_2}{\pi_1\big(\sigma_{\bm{\zeta}_1,m}^2+1\big) + (1-\pi_1)\big(\sigma_{\bm{\zeta}_2,m}^2+1\big)}
\end{equation}
We take a difference of the magnitudes of the NE based predictor and the ERM based predictor to get 

\begin{equation}
\begin{split}
 \frac{\big(1-\pi_1)\Big(\vartheta_1 \big(\sigma_{\bm{\zeta}_2,m}^2+1\big)-\vartheta_2\big(\sigma_{\bm{\zeta}_1,m}^2+1\big) \Big)}{\Big(\pi_1\big(\sigma_{\bm{\zeta}_1,m}^2+1\big) + (1-\pi_1)\big(\sigma_{\bm{\zeta}_2,m}^2+1\big)\Big)\big(\sigma_{\bm{\zeta}_1,m}^2+1\big)}
    \end{split}
    \label{diff_mag1}
\end{equation}
 Since  $\frac{\vartheta_2}{\sigma_{\bm{\zeta}_2,m}^2+1} <\frac{\vartheta_1}{\sigma_{\bm{\zeta}_1,m}^2+1}$ it  it follows that if $\pi_1\in [0,1)$, then the above difference in equation \eqref{diff_mag} is positive. However, if $\pi_1=1$, then the difference is zero. Therefore, except for when $\pi_1 = 1$, ERM is strictly worse than NE based ensemble predictor. This completes the analysis for all the possible cases. For each component where the least squares optimal solution differ, we showed that ERM based predictor is worse than NE based predictor except over a set of measure zero over the probability $\pi_1$. This completes the proof. 
 
\end{proof}

\subsection{Theorem \ref{brd:thm1}}
\label{sec:BRD_proof}
We restate Theorem \ref{brd:thm1} below. 

 \begin{theorem}
\label{brd:thm1_append}
If Assumption \ref{ass1}, \ref{realize_ass}, \ref{dis_ass} hold, then the output of Algorithm \ref{algorithm1}, $\bar{\bm{w}}^{+}$, is
\begin{equation*}
    \Big(\bm{w}_1^{*} \odot \bm{1}_{|\bm{w}_2^{*}|\geq |\bm{w}_1^{*}|} + \bm{w}_2^{*} \odot \bm{1}_{|\bm{w}_1^{*}|>|\bm{w}_2^{*}|} \Big) \bm{1}_{ \bm{w}_1^{*}\odot \bm{w}_2^{*} \geq \bm{0}}
    \label{ne_exp1_append}
\end{equation*}
\end{theorem}

\begin{proof}

From Theorem \ref{nash_char:thm}, we know that if Assumptions \ref{ass1}, \ref{realize_ass}, \ref{dis_ass} hold, then the NE based ensemble predictor is given as 
\begin{equation}
\Big(\bm{w}_1^{*} \odot \bm{1}_{|\bm{w}_2^{*}|\geq |\bm{w}_1^{*}|} + \bm{w}_2^{*} \odot \bm{1}_{|\bm{w}_1^{*}|>|\bm{w}_2^{*}|} \Big) \bm{1}_{ \bm{w}_1^{*}\odot \bm{w}_2^{*} \geq \bm{0}}
\label{eqn:conv_1}
\end{equation}

In Algorithm \ref{algorithm1}, each environment plays the optimal action given the action of the others. Hence, by best responding to each other we hope the procedure would converge to NE based ensemble predictor in equation \eqref{eqn:conv_1}.

 We write a dynamic which seems simpler than the dynamic in Algorithm \ref{algorithm1}. However, we show that the two are equivalent.
We index the iteration by $t$. Define $\bm{w}_e^{t} $ as the predictor for environment $e$ at the end of iteration $t$. The ensemble predictor is given as $\bar{\bm{w}}^{t} = \bm{w}_1^{t} + \bm{w}_2^{t}$. 
Each component of $e$'s predictor is given as $\bm{w}_{e}^{t} = (w_{e1}^{t},\dots, w_{en}^{t})$ and for the other environment $q\in \{1,2\}\setminus \{e\}$ as $\bm{w}_{-e}^{t} = (w_{-e1}^{t},\dots, w_{-en}^{t})$.  Environment $e$ in its turn sees that the other environment is using a predictor $\bm{w}_{-e}^{t-1}$; environment $e$ updates the predictor by taking a step such that $\min_{\bm{w}_{e}^{t}\in \mathcal{W}}\big\|\bm{w}_{e}^{t}+\bm{w}_{-e}^{t-1}-\bm{w}_{e}^{*}\big\|^2$, i.e. the environment moves such that it gets closest to the optimal least squares solution.  We can simplify this minimization as
\begin{equation}
\begin{split}
& \min_{\bm{w}_{e}^{t}\in \mathcal{W}}\big\|\bm{w}_{e}^{t}+\bm{w}_{-e}^{t-1}-\bm{w}_{e}^{*}\big\|^2 \\ & = \sum_{i=1}^{n}\min_{|w_{ei}^{t}|\leq w^{\mathsf{sup}}} (w_{ei}^{t} + w_{-ei}^{t} - w_{ei}^{*})^2
\end{split}
\label{eqn:thm3_proof_1}
\end{equation}

For $t=0$, $\bm{w}_{1}^{t} = \bm{0}$ and $\bm{w}_{2}^{t} = \bm{0}$. The dynamic based on equation \eqref{eqn:thm3_proof_1} is written as follows. 

For $t\geq 1$
\begin{equation}
        \bm{w}_{1}^{t} = \begin{cases}
                        \bm{w}_{1}^{t-1} \;\;\;\;\;\;\;\;\;\;\;\;\;\;\;\;\;\;\;\;\; \;t \text{ is even } \\
                        \Pi_{\mathcal{W}}[\bm{w}_{1}^{*} - \bm{w}_{2}^{t-1}] \;\;\;\;\; t \text{ is odd }
                    \end{cases}
                    \label{env1_update}
\end{equation}
\begin{equation}
     \bm{w}_{2}^{t} = \begin{cases}
                    \bm{w}_{2}^{t-1}, \;\;\;\;\;\;\;\;\;\;\;\;\;\;\;\;\;\;\;\; \;\; \;\;t \text{ is odd } \\
                        \Pi_{\mathcal{W}}[\bm{w}_{2}^{*} - \bm{w}_{1}^{t-1}] \;\;\;\;\;\;\;\; t \text{ is even }
                    \end{cases}
                    \label{env2_update}
\end{equation}

$\;\;\;\;\;\;\;\;t = t+1$

In the above equations \eqref{env1_update}, \eqref{env2_update}, $\Pi_{\mathcal{W}}$ represents the projection on the set $\mathcal{W} = \{\bm{w}\;\text{s.t.}\|\bm{w}\|_{\infty} \leq w^{\mathsf{sup}}\}$.

In each iteration, only one of the environment updates the predictors. In the above dynamic, whenever an environment completes its turn to update the predictor, $t$ is incremented by one. Before showing the convergence of this dynamic, we first need to establish that this dynamic is equivalent to the one stated in Algorithm, where when its the turn of environment $e$ to update it minimizes the following  $\min_{\bm{w}_{e}^{t}\in \mathcal{W}}R_{e}(\bm{w}_{e}^{t}, \bm{w}_{-e}^{t-1})$. 

\textbf{Equivalence of dynamic in equations \eqref{env1_update} and \eqref{env2_update} to dynamic in Algorithm \ref{algorithm1}}

Recall from Section \ref{secn:clrg} and proof of Theorem \ref{nash_char:thm}, that we divide the feature components $\{1,\dots,n\}$ into two sets, the first $k$ components are in $\mathcal{U}$ and the next $n-k$ components are in $\mathcal{V}$. The two environments have the same least squares coefficients for components in $\mathcal{U}$ but have differing coefficients for points in $\mathcal{V}$. For each $e\in \{1,2\}$,  $\bm{w}_{e+}^{t}$ corresponds to the first $k$ coefficient at the end of iteration $t$ and $\bm{w}_{e-}^{t}$ corresponds to the next $n-k$ coefficients at the end of iteration $t$. 

Recall the decomposition that we stated in equation \eqref{eqn:thm2_proof_5}. To arrive at the equation \eqref{eqn:thm2_proof_5} we used Assumptions \ref{ass1} and \ref{dis_ass}. We continue to make these assumptions in this theorem as well. Therefore, we can continue to use the decomposition in equation \eqref{eqn:thm2_proof_5}. For  environment $2$ we can write
\begin{equation}
\begin{split}
    & \min_{\bm{w}_2^{t}\in \mathcal{W}} R_{2}(\bm{w}_1^{t-1},\bm{w}_2^{t}) = \\
    & = \min_{\bm{w}_{2+}^{t}\in \mathcal{W}_{+}}R_{2+}(\bm{w}_{1+}^{t-1},\bm{w}_{2+}^{t-1})  +\sum_{i}\sigma_{2i}^2\min_{|w_{2i}^{t}|\leq w^{\mathsf{sup}}}(w_{2i}^{t} + w_{1i}^{t} - w_{2i}^{*})^2
\end{split}
\label{eqn:thm3_proof_2}
\end{equation}

A decomposition identical to above equation \eqref{eqn:thm3_proof_2} also holds for environment $1$.  From equation \eqref{eqn:thm3_proof_1} and \eqref{eqn:thm3_proof_2}, we gather that for latter $n-k$ components, the update rule in equations \eqref{env1_update}, \eqref{env2_update} and the update rule in Algorithm \ref{algorithm1} are equivalent for both the environments.  We now show that both the rules are equivalent in the first $k$ components as well.

Consider iteration $t=1$. If the environment $1$ uses the update rule in equation \eqref{env1_update}, then the ensemble predictor is set to $\bm{w}_1^{*}$ (From Assumption $\bm{w}_1^{*} \in \mathcal{W}$ is in the interior, the projection will be the point itself).
Note that if environment $1$ used $\min_{\bm{w}_{1}^{t}\in \mathcal{W}}R_{1}(\bm{w}_{1}^{t}, \bm{w}_{2}^{t-1})$, then as well it will move the ensemble predictor to $\bm{w}_1^{*}$ (since $\bm{w}_{1}^{*}$ is the least squares optimal solution).  

Consider iteration $t=2$.  Suppose the environment $2$ uses  the update rule in equation \eqref{env2_update}. Define $\bm{\Delta}^{*}=\bm{w}_2^{*}-\bm{w}_1^{*}$ and the $i^{th}$ component of $\bm{\Delta}^{*}$ as $\Delta_i^{*}$. Given the environment $1$ is at $\bm{w}_1^{*}$ the rule dictates that environment $2$ should update the predictor to $\Pi_{\mathcal{W}}[\bm{\Delta}^{*}]$.  The first $k$ components of $\bm{\Delta}^{*}$ would be zero as the two environments agree in these coefficients. Therefore, environment $2$ will not move its predictor for the first $k$ components and continue to be at $\bm{0}$. After this the two environments do not need to update the first $k$ components as they have already converged. Suppose the environment $2$ uses the update rule from Algorithm \ref{algorithm1}.  Consider the first $k$ components in which  both the environments agree. Since environment $1$ already is using $\bm{w}_{1+}^{*}$, which is optimal for environment $2$ as well, environment $2$ will not move its predictor for first $k$ components and continue to be at $\bm{0}$. After this the two environments do not need to update the first $k$ components as they have already converged.

Thus so far we have established that both dynamics in equations \eqref{env1_update}, \eqref{env2_update} and the update rule in Algorithm \ref{algorithm1} are equivalent. We have also shown the convergence in the first $k$ components. We now focus on establishing the convergence for the next $n-k$ components that make up the set $\mathcal{V}$.

\textbf{Convergence of the dynamics in equations \eqref{env1_update} and \eqref{env2_update}.}
In the previous section, we showed that dynamic in equations \eqref{env1_update} and \eqref{env2_update} are equivalent to the dynamic in Algorithm \ref{algorithm1}. While we had shown the convergence of the first $k$ components in the set $\mathcal{U}$, we will repeat the analysis for ease of exposition. In this section, we begin by showing how the dynamic in equations \eqref{env1_update} and \eqref{env2_update} plays out. Just for the sake of clarity of exposition, in the  dynamic we show below we assume that the predictors of the environment continue to be in the interior of the set $\mathcal{W}$. 

\begin{enumerate}
    \item End of $t=1$, $\bar{\bm{w}}^{t}=\bm{w}_1^{*}$, $e=1$ plays $\bm{w}_1^{t}=\bm{w}_1^{*}$, $e=2$ plays  $\bm{w}_2^{t}=\bm{0}$. 
    \item End of $t=2$, $\bar{\bm{w}}^{t}=\bm{w}_2^{*}$, $e=1$ plays $\bm{w}_1^{t}=\bm{w}_1^{*}$, $e=2$ plays  $\bm{w}_2^{t}=\bm{\Delta}^{*}$. 
    \item End of $t=3$, $\bar{\bm{w}}^{t}=\bm{w}_1^{*}$, $e=1$ plays $\bm{w}_1^{t}=\bm{w}_1^{*}-\bm{\Delta}^{*}$, $e=2$ plays  $\bm{w}_2^{t}=\bm{\Delta}^{*}$. 
      \item End of $t=4$, $\bar{\bm{w}}^{t}=\bm{w}_2^{*}$, $e=1$ plays $\bm{w}_1^{t}=\bm{w}_1^{*}-\bm{\Delta}^{*}$, $e=2$ plays  $\bm{w}_2^{t}=2\bm{\Delta}^{*}$. 
      \item End of $t=5$, $\bar{\bm{w}}^{t}=\bm{w}_1^{*}$, $e=1$ plays $\bm{w}_1^{t}=\bm{w}_1^{*}-2\bm{\Delta}^{*}$, $e=2$ plays  $\bm{w}_2^{t}=2\bm{\Delta}^{*}$. 
    \item End of $t=6$, $\bar{\bm{w}}^{t}=\bm{w}_2^{*}$, $e=1$ plays $\bm{w}_1^{t}=\bm{w}_1^{*}-2\bm{\Delta}^{*}$, $e=2$ plays  $\bm{w}_2^{t}=3\bm{\Delta}^{*}$. 
\end{enumerate}

In the dynamic displayed above, we assumed that the predictor $\bm{w}_{1}^{t}$ and $\bm{w}_2^{t}$ were in the interior just to illustrate that the two sequences $\bm{w}_{1}^{t}$ and $\bm{w}_2^{t}$ are monotonic. Observe that if a certain component of $\bm{\Delta}^{*}$ say $\Delta_i^{*}$ is non-zero, the two sequences are strictly monotonic in that component. The sequences cannot grow unbounded and at least one of them will first hit the boundary at $w^{\mathsf{sup}}$ or $-w^{\mathsf{sup}}$. Recall from the last section, where we already showed that for the first $k$ components of $\bm{\Delta}^{*}$ associated with these $\mathcal{U}$ are zero (from equation \eqref{eqn:thm2_proof_6_n}). Hence, for the first $k$ components the the dynamic $\bm{w}_{1}^{t}$ and $\bm{w}_2^{t}$ converges at the end of $t=1$.  Therefore, we now only need to focus on the remaining $n-k$ components comprising the set $\mathcal{V}$. Since the update rules in equation \eqref{env1_update} and \eqref{env2_update} are separable for the different components, we only focus on one of the components say $i$.

 We divide our analysis based on if $w_{1i}^{*}$ and $w_{2i}^{*}$ have the same sign or not. Suppose $w_{1i}^{*}$ and $w_{2i}^{*}$ have the same sign. Let us consider the case when both are positive (negative case follows from symmetry as the dynamic starts at zero). 
 \begin{itemize}
 \item Suppose $0 \leq w_{1i}^{*}< w_{2i}^{*}$. If  $0 \leq w_{1i}^{*}< w_{2i}^{*}$ is plugged into the equation \eqref{ne_exp}, we obtain $w_{1i}^{*}$. Our objective is to show convergence to $w_{1i}^{*}$.   In this case, $\Delta_{i}^{*}$, which corresponds to the $i^{th}$ component of $\bm{\Delta}^{*}$, is greater than zero. Observe from the dynamic that environment $2$ will first hit the boundary in this case and since $\Delta_{i}^{*}>0$, it will hit the positive end, i.e., $w^{\mathsf{sup}}$. The best response of environment $1$ is to play $\Pi_{\mathcal{W}}[w_{1i}^{*}-w^{\mathsf{sup}}]$. Since $w_{1i}^{*}>0$, we get that environment $1$ uses the predictor $w_{1i}^{*}-w^{\mathsf{sup}}$, the ensemble predictor takes the value $w_{1i}^{*}$. Environment $2$ in the next step continues to play $\Pi_{\mathcal{W}}[w_{2i}^{*}-w_{1i}^{*}+ w^{\mathsf{sup}}] = w^{\mathsf{sup}}$ and environment $1$ continues to play $w_{1i}^{*}-w^{\mathsf{sup}}$. Hence, the predictors stop updating. Thus in this case at convergence, the ensemble classifier achieves the value that we wanted to prove $w_{1i}^{*}$. Also, both environments best respond to each other, which implies that the state is a NE. 
 \item Suppose $0 \leq w_{2i}^{*}< w_{1i}^{*}$. If  $0 \leq w_{2i}^{*}< w_{1i}^{*}$ is plugged into the equation \eqref{ne_exp}, we obtain $w_{1i}^{*}$. Our objective is to show convergence to $w_{2i}^{*}$.  Observe from the dynamic that environment $1$ will first hit the boundary in this case and since $\Delta_{i}^{*}<0$, it will hit the positive end, i.e., $w^{\mathsf{sup}}$. The best response of environment $2$ is to play $\Pi_{\mathcal{W}}[w_{2i}^{*}-w^{\mathsf{sup}}]$. Since $w_{2i}^{*}>0$, we get that environment $2$ uses the  predictor $w_{2i}^{*}-w^{\mathsf{sup}}$ and the ensemble predictor takes the value $w_{2i}^{*}$. Thus just like the case described above both environments stop updating. Hence, the state $w^{\mathsf{sup}}, w_{2i}^{*}-w^{\mathsf{sup}}$ is a NE and the final ensemble predictor is at $w_{2i}^{*}$, which is what we wanted to prove. 
 \end{itemize}
 
 Suppose $w_{1i}^{*}$ and $w_{2i}^{*}$ have the opposite sign.
 
 \begin{itemize}
     \item Consider the case when $w_{1i}^{*}< 0<w_{2i}^{*}$. If  $w_{1i}^{*}< 0 < w_{2i}^{*}$ is plugged into the equation \eqref{ne_exp}, we obtain $0$. In this setting, environment $2$ moves towards the $w^{\mathsf{sup}}$ and environment $1$ moves towards $-w^{\mathsf{sup}}$. Suppose environment $2$ hits the boundary. In the next step, the best response from environment $1$ is computed as $\Pi_{\mathcal{W}}[w_{1i}^{*}-w^{\mathsf{sup}}] = -w^{\mathsf{sup}}$. Since both environments best respond to each other the state $(-w^{\mathsf{sup}}, w^{\mathsf{sup}})$ is NE and the final ensemble predictor is at $0$, which is what we wanted to prove. Hence, both the environment continue to stay at the boundary. This is also the case if environment $1$ hits the boundary first. The same analysis applies to the case when $w_{2i}^{*}< 0<w_{1i}^{*}$. 
 \end{itemize}

We focused on one of the components $i$ and the above analysis applies to all the components $j\in \{k+1,..\dots, n\}$. This completes the proof.  Note that the entire analysis is symmetric and it does not matter which environment moves first, the analysis also extends to the case when initialization is not zero.  
\end{proof}

\subsection{Proposition 5}
We restate Proposition \ref{prop5}
\begin{proposition}
\label{prop5_append}
If Assumption \ref{linear_model_ass} holds with $\bm{\alpha}_e=\bm{0}$ and $\bm{\Theta}_e$ an  orthogonal matrix for each $e\in \{1,2\}$, and Assumptions \ref{ass2}, \ref{sp_var}, \ref{realize_ass} hold, then the output of Algorithm \ref{algorithm1}, $\bar{\bm{w}}^{+}$ obeys  $\|\bar{\bm{w}}^{+} - (\bm{\gamma},0)\| < \|\bm{w}^{\mathsf{ERM}} - (\bm{\gamma},0)\|$ for all $\bm{w}^{\mathsf{ERM}} \in \mathcal{S}^{\mathsf{ERM}}$ except over a set of measure zero (see footnote 2).  Moreover, if all the components of vectors $\bm{\Theta}_1\bm{\eta}_1$ and $\bm{\Theta}_2\bm{\eta}_2$ have opposite signs, then $\bar{\bm{w}}^{+}  = (\bm{\gamma},0)$.
\end{proposition}

\begin{proof}
In the above proposition, we make the same set of assumptions as in Proposition \ref{prop4}.
In the proof of Proposition \ref{prop4}, we showed that the set of Assumptions in Proposition \ref{prop4} imply that the Assumptions  \ref{ass1}, \ref{realize_ass}, \ref{dis_ass} hold. Since Assumptions \ref{ass1}, \ref{realize_ass}, \ref{dis_ass} hold, from Theorem \ref{brd:thm1} it follows that the output of Algorithm \ref{algorithm1} is equal to the NE based ensemble predictor  given by equation \eqref{ne_exp}. In Proposition \ref{prop4}, we have already shown that this NE based ensemble predictor (equation \eqref{ne_exp}), which is the output of Algorithm \ref{algorithm1},  is closer to $(\bm{\gamma},\bm{0})$ than the solution of ERM (except over a set of measure zero defined in the proof of Proposition \ref{prop4}). We had also shown that the NE based ensemble predictor is equal $(\bm{\gamma},\bm{0})$ when the signs of $\bm{\Theta}_1\bm{\eta}_1$ and $\bm{\Theta}_2\bm{\eta}_2$ are opposite.

This completes the proof.
\end{proof}

\subsection{Other best response dynamics} 
\label{sec:o_BRD}

In this section, we describe a simple signed gradient descent based dynamic. The aim is to show that for simple variations of the dynamic proposed in Algorithm \ref{algorithm1} we  continue to have convergence guarantees.

We define a signed gradient descent based version of the dynamic in equation \eqref{env1_update} and
\eqref{env2_update} with step length $\beta$.

For $t=0$, $\bm{w}_{1}^{t} = \bm{0}$ and $\bm{w}_{2}^{t} = \bm{0}$.

For $t\geq 1$
\begin{equation}
        \bm{w}_{1}^{t} =
                        \Pi_{\mathcal{W}} \Big[\bm{w}_{1}^{t-1} + \beta \mathsf{sgn}[\bm{w}_{1}^{*} - \bm{w}_{2}^{t-1}]\Big] \;\;\;\;\; 
                    \label{env1_update_n}
\end{equation}
\begin{equation}
        \bm{w}_{2}^{t} =
                        \Pi_{\mathcal{W}} \Big[\bm{w}_{2}^{t-1} + \beta \mathsf{sgn}[\bm{w}_{2}^{*} - \bm{w}_{1}^{t-1}]\Big] \;\;\;\;\; 
                    \label{env2_update_n}
\end{equation}

$\;\;\;\;\;\;\;\;\bar{\bm{w}}^{t} = \bm{w}_{1}^{t}+ \bm{w}_{2}^{t}$

$\;\;\;\;\;\;\;\;t=t+1$

In the above $\mathsf{sgn}$ is the component-wise sign function, which takes a value $1$ when the input is positive (including zero) and $-1$ if the input is negative. $\bar{\bm{w}}^{t}$ is the ensemble predictor at the end of iteration $t$, $\bm{w}_e^{t} $ is the predictor for environment $e$ at the end of iteration $t$.
Suppose we want to get within $\epsilon$ (per component) distance of the NE based predictor. 
 Divide the components into two sets $\mathcal{E}$ and $\mathcal{F}$ defined as follows.  $i\in \mathcal{E}$ if and only if the least squares solution are within $\epsilon$ distance, i.e. $|w_{1i}^{*}-w_{2i}^{*}|\leq \epsilon$ and $i \in \mathcal{F}$ if and only if the least squares solution are separated by at least epsilon i.e. $|w_{1i}^{*}-w_{2i}^{*}|> \epsilon$. 
 
 Let $\beta <\epsilon$.  Let us analyze the dynamic for a component $i \in 
\mathcal{F}$. 
We divide the analysis into two cases -- $w_{1i}^{*}$ and $w_{2i}^{*}$ have the same sign, and $w_{1i}^{*}$ and $w_{2i}^{*}$ have opposite signs.  Let us start with same sign case with both $w_{1i}^{*}$ and $w_{2i}^{*}$ positive (negative sign case follows from symmetry). Consider the case $0<w_{1i}^{*}< w_{2i}^{*}$.  In this case from the expression of NE in equation \eqref{ne_exp}, we would hope the dynamic can eventually achieve an ensemble predictor that stays within $\epsilon$ distance of $w_{1i}^{*}$. 
Since the dynamic starts at $0$ and $\mathsf{sgn}[w_{1i}^{*}]=1$ and $\mathsf{sgn}[w_{2i}^{*}]=1$, the ensemble predictor $\bar{w}_i^{t}$ after some iterations will enter the interval  $[w_{1i}^{*}, w_{2i}^{*}]$. Once the ensemble predictor enters the interval, the two predictors $w_{1i}^{t}$ and $w_{2i}^{t}$ will push the predictor in opposite directions and since the step length is the same the ensemble predictor will not move. This would continue until environment $2$ hits the positive boundary $w^{\mathsf{sup}}$. Once the environment $2$ hits the boundary it stops updating and environment $1$ pushes the ensemble predictor towards $w_{1i}^{*}$. Once the predictor is within $ \beta$ distance from $w_{1i}^{*}$ it continues to oscillate around $w_{1i}^{*}$. Consider the case $0<w_{2i}^{*}< w_{1i}^{*}$, the same analysis follows and dynamic eventually oscillates around $w_{2i}^{*}$.
Next, consider the case when $w_{1i}^{*}$ and $w_{2i}^{*}$ have opposite signs. In this case from the expression of NE in \eqref{ne_exp}, we would hope the dynamic can eventually achieve an ensemble predictor that stays within $\epsilon$ distance of $0$.  In this case, since the predictors start at zero, both environments will push in opposite directions. Eventually, both environments hit opposite ends of the boundary and stay there. This results in ensemble predictor coefficient of zero.

Now let us analyze the game for components in $\mathcal{E}$. We again carry out analysis based on the whether the signs agree or not. Consider the case $0<w_{1i}^{*}< w_{2i}^{*}$.  In this case from the expression of NE in \eqref{ne_exp}, we would hope the dynamic can eventually achieve an ensemble predictor that stays within $\epsilon$ distance of $w_{1i}^{*}$.  The dynamic starts at $0$ and $\mathsf{sgn}[w_{1i}^{*}]=1$ and $\mathsf{sgn}[w_{2i}^{*}]=1$, the ensemble predictor $\bar{w}_i^{t}$ may or may not enter the interval $[w_{1i}^{*}, w_{2i}^{*}]$. If it enters, then the analysis is identical to the previous case when $i\in \mathcal{F}$. If the ensemble predictor does not enter the interval, then it has to overshoot it and move to the right of the interval, which implies in the next step it will be pulled back to the left of the interval. This sets the ensemble predictor in an oscillation around $w_{1i}^{*}$. The analysis for the case when  $w_{1i}^{*}$ and $w_{2i}^{*}$ have opposite signs is the same as the previous case.

\subsection{Convergence when Assumption \ref{dis_ass} does not hold}

In this section, we discuss can we still learn NE if the Assumption \ref{dis_ass} is relaxed? We would rely on the results in \cite{zhou2017mirror} for our discussion here. In \cite{zhou2017mirror}, the authors introduced a notion called variational stability. Consider the class of concave games. It was shown that if the set of Nash equilibria satisfy variational stability, then a mirror descent based learning dynamic (described in \cite{zhou2017mirror}) converges to the NE.  Next, we analyze the variational stability for C-LRG.

Define the gradient of utility of the environment $e$  
$\bm{v}_{e}(\bm{w}) = -\nabla_{\bm{w}_{e}}R_{e}(\bm{w}_e,\bm{w}_{-e})$, where recall that $\bm{w}_e$ is action of environment $e$ and $\bm{w}_{-e}$ is the action of the other environment, $R_e$ is the risk, and $\bm{w}= (\bm{w}_{e}, \bm{w}_{-e})$.  
Let us recall a characterization of NE in terms of the gradients (\cite{zhou2017mirror}). 
Suppose $\bm{w}^{\dagger} = (\bm{w}_1^{\dagger}, \bm{w}_2^{\dagger})$ is a NE of C-LRG. For every $\bm{w}_e\in \mathcal{W}$, we have
\begin{equation}
    \bm{v}_{e}(\bm{w}^{\dagger})^{\mathsf{T}}(\bm{w}_e-\bm{w}_e^{\dagger})  \leq 0
    \label{ne_prop1}
\end{equation}

Next, we show how to relate the gradient at NE, $\bm{v}_{e}(\bm{w}^{\dagger})$,  to the gradient at any other point $\bm{v}_{e}(\bm{w})$
\begin{equation}
\begin{split}
  & \bm{v}_{e}(\bm{w}) = \mathbb{E}_{e}\Big[\big(Y_{e}-\bm{w}_1^{\mathsf{T}}\bm{X}_e -\bm{w}_2^{\mathsf{T}}\bm{X}_{e}\big)\bm{X}_e\Big] \\ 
 & = \mathbb{E}_{e}\Big[\Big(Y_{e}-(\bm{w}_1-\bm{w}_1^{\dagger} + \bm{w}_1^{\dagger})^{\mathsf{T}}\bm{X}_e -(\bm{w}_2-\bm{w}_2^{\dagger} + \bm{w}_2^{\dagger})^{\mathsf{T}}\bm{X}_{e}\Big)\bm{X}_e\Big] \\
 & \bm{v}_e(\bm{w}) = \bm{v}_e(\bm{w}^{\dagger}) - \bm{\Sigma}_e(\bm{w}_1-\bm{w}_1^{\dagger} + \bm{w}_2 - \bm{w}_2^{\dagger})\\   & \bm{v}_e(\bm{w}) = \bm{v}_e(\bm{w}^{\dagger}) - \bm{\Sigma}_e(\bar{\bm{w}} -\bar{\bm{w}}^{\dagger})  
\end{split}
\end{equation}
In the above $\bar{\bm{w}} = \bm{w}_1 + \bm{w}_2$, $\bar{\bm{w}}^{\dagger} = \bm{w}_1^{\dagger} + \bm{w}_2^{\dagger}$. 

For establishing variational stability, we  need to show that  for each $\bm{w} \in \mathcal{W}\times \mathcal{W}$ and for each NE $\bm{w}^{\dagger}$ the following inequality, i.e., $\sum_{e} \bm{v}_{e}(\bm{w})^{\mathsf{T}}(\bm{w}_e-\bm{w}_e^{\dagger})\leq 0$, holds.

\begin{equation}
\begin{split}
 &   \sum_{e\in \{1,2\}} \bm{v}_{e}(\bm{w})^{\mathsf{T}}(\bm{w}_e-\bm{w}_e^{\dagger}) =    \sum_{e\in \{1,2\}} \bm{v}_{e}(\bm{w}^{\dagger})^{\mathsf{T}}(\bm{w}_e-\bm{w}_e^{\dagger}) - \sum_{e\in \{1,2\}}(\bar{\bm{w}} -\bar{\bm{w}}^{\dagger})^{\mathsf{T}} \bm{\Sigma}_e(\bm{w}_e-\bm{w}_e^{\dagger}) \\ 
 & =\sum_{e\in \{1,2\}} \bm{v}_{e}(\bm{w}^{\dagger})^{\mathsf{T}}(\bm{w}_e-\bm{w}_e^{\dagger}) - (\bar{\bm{w}} -\bar{\bm{w}}^{\dagger})^{\mathsf{T}} \bm{\Sigma}_1(\bar{\bm{w}}-\bar{\bm{w}}^{\dagger})   -(\bar{\bm{w}}-\bar{\bm{w}}^{\dagger})^{\mathsf{T}} (\bm{\Sigma}_2-\bm{\Sigma}_1)(\bm{w}_2-\bm{w}_2^{\dagger})
\end{split}
\label{vs_simp1}
\end{equation}

 We begin by analyzing the case when $\bm{\Sigma}_1 = \bm{\Sigma}_2$. Substitute $\bm{\Sigma}_1 = \bm{\Sigma}_2$ in equation \eqref{vs_simp1}, 

\begin{equation*}
\begin{split}
  & \sum_{e\in \{1,2\}} \bm{v}_{e}(\bm{w})^{\mathsf{T}}(\bm{w}_e-\bm{w}_e^{\dagger})= \\ &
  \sum_{e\in \{1,2\}} \bm{v}_{e}(\bm{w}^{\dagger})^{\mathsf{T}}(\bm{w}_e-\bm{w}_e^{\dagger}) - (\bar{\bm{w}} -\bar{\bm{w}}^{\dagger})^{\mathsf{T}} \bm{\Sigma}_1(\bar{\bm{w}}-\bar{\bm{w}}^{\dagger}) 
  \end{split}
\end{equation*}

If we use the condition in equation \eqref{ne_prop1} along with the fact that $\bm{\Sigma}_1$ is positive definite, then we get that 
$\sum_{e\in \{1,2\}} \bm{v}_{e}(\bm{w})^{\mathsf{T}}(\bm{w}_e-\bm{w}_e^{\dagger}) \leq 0$, which implies that the set of NE of C-LRG is variationally stable. 
We now give an example of when  $\bm{\Sigma}_1 = \bm{\Sigma}_2$ is satisfied. Consider the SEM in Assumption  \ref{linear_model_ass}. If between the two environments the only parameters that vary are $\bm{\eta}_e$ and the distribution of $\varepsilon_e$, and rest all other parameters in the model are the same, then $\bm{\Sigma}_1 = \bm{\Sigma}_2$ is satisfied.

We now discuss what happens if we relax the assumption, $\bm{\Sigma}_1 = \bm{\Sigma}_2$, made above. 

Consider the eigenvalue decomposition of $\bm{\Sigma}_2-\bm{\Sigma}_1 = \bm{\Omega}\bm{\Lambda}\bm{\Omega}^{\mathsf{T}}$, where since $\bm{\Sigma}_2-\bm{\Sigma}_1$ is a symmetric matrix we know that $\bm{\Omega}$ can be chosen as an orthogonal matrix and $\bm{\Lambda}$ is a diagonal matrix of eigenvalues. Define the smallest eigenvalue of $\bm{\Sigma}_2-\bm{\Sigma}_1$ as $\lambda_{\mathsf{min}}(\bm{\Sigma}_2 - \bm{\Sigma}_1)$. 

Define a transformation of vector $\bm{w}$ under $\bm{\Omega}$ as  $\tilde{\bm{w}} = \bm{\Omega} \bm{w}$. Since $\bm{\Omega}$ is orthogonal, $\|\tilde{\bm{w}}\| = \|\bm{w}\|$. 
We now use these relationships to simplify

\begin{equation}
\begin{split}
    &  (\bm{w}_1-\bm{w}_1^{\dagger})^{\mathsf{T}} (\bm{\Sigma}_2-\bm{\Sigma}_1)(\bm{w}_2-\bm{w}_2^{\dagger}) =(\bm{w}_1-\bm{w}_1^{\dagger})^{\mathsf{T}}\bm{\Omega}\bm{\Lambda}\bm{\Omega}^{\mathsf{T}} (\bm{w}_2-\bm{w}_2^{\dagger})  = (\tilde{\bm{w}}_1-\tilde{\bm{w}}_1^{\dagger})\bm{\Lambda}(\tilde{\bm{w}}_2-\tilde{\bm{w}}_2^{\dagger}) \\
   &   \geq \lambda_{\mathsf{min}}(\bm{\Sigma}_2-\bm{\Sigma}_1)(\tilde{\bm{w}}_1-\tilde{\bm{w}}_1^{\dagger})^{\mathsf{T}}(\tilde{\bm{w}}_2-\tilde{\bm{w}}_2^{\dagger}) \\ 
   & \geq 
   -|\lambda_{\mathsf{min}}(\bm{\Sigma}_2-\bm{\Sigma}_1)|\|(\tilde{\bm{w}}_1-\tilde{\bm{w}}_1^{\dagger})\|
   \|(\tilde{\bm{w}}_2-\tilde{\bm{w}}_2^{\dagger})\| \\
    & = -|\lambda_{\mathsf{min}}(\bm{\Sigma}_2-\bm{\Sigma}_1)|\|(\bm{w}_1-\bm{w}_1^{\dagger})\|\|(\bm{w}_2-\bm{w}_2^{\dagger})\|
\end{split}
\label{vs_simp2}
\end{equation}
In the last two inequalities in equation \eqref{vs_simp2}, we used Cauchy-Schwarz inequality and the fact that norms do not change under orthogonal transformations $\|\tilde{\bm{w}}\| = \|\bm{w}\|$. Now let us bound the term $(\bar{\bm{w}}-\bar{\bm{w}}^{\dagger})^{\mathsf{T}} (\bm{\Sigma}_2-\bm{\Sigma}_1)(\bm{w}_2-\bm{w}_2^{\dagger})$ in equation \eqref{vs_simp1}.

\begin{equation}
\begin{split}
   & -(\bar{\bm{w}}-\bar{\bm{w}}^{\dagger})^{\mathsf{T}} (\bm{\Sigma}_2-\bm{\Sigma}_1)(\bm{w}_2-\bm{w}_2^{\dagger})  = \\& -(\bm{w}_2-\bm{w}_2^{\dagger})^{\mathsf{T}} (\bm{\Sigma}_2-\bm{\Sigma}_1)(\bm{w}_2-\bm{w}_2^{\dagger}) + -(\bm{w}_1-\bm{w}_1^{\dagger})^{\mathsf{T}} (\bm{\Sigma}_2-\bm{\Sigma}_1)(\bm{w}_2-\bm{w}_2^{\dagger}) \\ 
    &  \leq -\lambda_{\mathsf{min}}(\bm{\Sigma}_2-\bm{\Sigma}_1)\|(\bm{w}_2-\bm{w}_2^{\dagger})\|^2 + |\lambda_{\mathsf{min}}(\bm{\Sigma}_2-\bm{\Sigma}_1)|\|(\bm{w}_1-\bm{w}_1^{\dagger})\|\|(\bm{w}_2-\bm{w}_2^{\dagger})\| 
\end{split}
\label{vs_simp3}
\end{equation}
In the last inequality in equation \eqref{vs_simp3}, we used equation \eqref{vs_simp2}. If $\bm{\Sigma}_2-\bm{\Sigma}_1$ is positive semi-definite with  lowest eigenvalue of zero, then the term in equation \eqref{vs_simp3} is bounded above by zero. If we use this observation in equation \eqref{vs_simp1}, the the condition for variational stability is satisfied. Note that the entire analysis is symmetric and we can state the same result for the matrix $\bm{\Sigma}_1-\bm{\Sigma}_2$. Therefore, if one of the matrices $\bm{\Sigma}_2-\bm{\Sigma}_1$ or $\bm{\Sigma}_1-\bm{\Sigma}_2$, is positive semi-definite with  lowest eigenvalue of zero, then we get variational stability for the NE. As a result,  we can use the convergence results in \cite{zhou2017mirror} to guarantee that NE will be learned.

\subsection{Extensions}

\subsubsection{Multiple Environments} 
In the main body of the paper, we discussed the results when the data is gathered from two environments. What happens if the data were gathered from multiple (more than two) environments? 

First let us start with the game U-LRG from Section \ref{senc:ulrg}. The first result in Proposition \ref{prop1} states that when least squares optimal solution are not equal, then there is no NE of U-LRG. When we move to multiple environments using same proof techniques it can be shown that if there is any two environments, which do not agree on the least squares optimal solution, then no NE will exist. For a NE to exist all environments will have to have the same least squares optimal solution. 

Next, we consider the game C-LRG from Section \ref{secn:clrg}. With multiple (more than two environments), we are guaranteed that NE will exist. How does the Theorem \ref{nash_char:thm} change for multiple environments? We  extend the Assumption \ref{dis_ass} to state that any feature component that does not have the same least squares coefficient across all the environments is uncorrelated with the rest of the features. Suppose the environments are indexed from $\{1,\dots, r\}$. 
For this discussion, let us focus on one of the feature components say $i$. Without loss of generality, assume that these environments are ordered in an increasing order w.r.t the optimal least squares coefficient, i.e. if $e, f\in \{1,\dots, r\}$ such that $e\leq f$, then $w_{ei}^{*} \leq w_{fi}^{*}$.

Let us assume that $r$ is odd. Consider the median environment indexed $m = \frac{r+1}{2}$.  Ensemble predictor's coefficient will be equal to the coefficient of the median environment $\bar{w}_{i}^{\dagger} = w_{mi}^{*}$. In this case in the NE, all the environments with index $e>m$ play $w^{\mathsf{sup}}$, all the environments with index $e<m$ play $-w^{\mathsf{sup}}$, and median environment $m$ plays $w_{mi}^{*}$. 

Let us assume that $r$ is even.  Consider the two median environments indexed $m = \frac{r}{2}$ and $m+1$. If $w_{mi}^{*}$ and $w_{(m+1)i}^{*}$ have the same sign, then the NE based ensemble predictor is equal to the coefficient with a smaller absolute value.  If $w_{mi}^{*}$ and $w_{(m+1)i}^{*}$ have the same sign, and say $0 \leq w_{mi}^{*} \leq w_{(m+1)i}^{*}$, then in NE the environment $m$ plays $w_{mi}^{*}-w^{\mathsf{sup}}$ and environment $m+1$ plays $w^{\mathsf{sup}}$. If $w_{mi}^{*}$ and $w_{(m+1)i}^{*}$ have the opposite sign, then the NE based ensemble predictor is equal to zero. If $w_{mi}^{*}$ and $w_{(m+1)i}^{*}$ have the opposite sign, and say $ w_{mi}^{*}<0 \leq w_{(m+1)i}^{*}$, then in NE the environment $m$ plays $-w^{\mathsf{sup}}$ and environment $m+1$ plays $w^{\mathsf{sup}}$.
For all the remaining environments other than $m$ and $m+1$ their actions are described as --- environments with index $e>m+1$ play $w^{\mathsf{sup}}$, all the environments with index $e<m$ play $-w^{\mathsf{sup}}$. 

In Proposition \ref{prop4}, we analyzed linear SEMs and showed that NE based ensemble predictor are closer to the OOD solutions than ERM. Proposition \ref{prop4} relied on Theorem \ref{nash_char:thm}, which we have shown can be appropriately extended to multi-environment setting. Hence, by using the same proof techniques used to prove Proposition \ref{prop4} and the expression for NE that we discussed above for multiple environments, we can show that the same result extends to multiple environments.  

In Theorem \ref{brd:thm1}, we proved the convergence for BRD dynamics to NE. We can show convergence in this setup as well using the same ideas discussed in Section \ref{sec:BRD_proof} and Section \ref{sec:o_BRD}.  We have shown that Theorem \ref{nash_char:thm}, Proposition \ref{prop4} and Theorem \ref{brd:thm1} extend to multi-environment setting. We also know that Proposition \ref{prop5} directly follows from Theorem \ref{nash_char:thm}, Proposition \ref{prop4} and Theorem \ref{brd:thm1}. Therefore, Proposition \ref{prop5} extends to multi-environment setting.

\subsection{Non-linear Models}
In the main body of the paper, we discussed all the results for linear models. An important direction for future work is to consider non-linear models.  A natural extension of the linear models is to consider models from reproducing kernel Hilbert spaces (RKHS). In the standard analysis of kernel regressions, the optimal predictors can be expressed as  linear functions of the kernel evaluated at the different data points. Similar to the constrained linear regression game that we introduced in Section \ref{secn:clrg}, we can introduce a constrained kernel regression game, where each environment's model choice is a predictor from RKHS with a constraint on its norm. Owing to the similarity of analysis of kernel regression and standard linear regression, we believe that the type of analysis that we carried out for linear regression games can serve as a useful building block to solving kernel regression games.

\subsection{Supplement for Experiments}

\subsubsection{Computing Environment}
The experiments were done on 2.3 GHZ Intel Core i9 processor with 32 GB
memory (2400 MHz DDR4). 

\subsubsection{Model, Hyperparameter, Training details}
We use linear models for all the methods. We carried out 10 trials for all the experiments and show the average performance in Figure \ref{fig2}. In the experiments for the 10 dimensional case ($p=q=5$) shown in Figure \ref{fig2}, we use a bound of $w^{\mathsf{sup}}=2$. In Figure \ref{fig1}, we had shown that provided the solution is contained in the search space, i.e., realizability assumption (Assumption \ref{realize_ass}) is satisfied, then the choice of the bound does not impact the solution provided the number of training steps are sufficiently large.

We use a stochastic gradient descent based best response dynamic to learn the NE; this dynamic is very similar to the one described in Section \ref{sec:o_BRD}.  For each environment $e\in\{1,2\}$ say the loss for the current batch at the end of iteration $t$ is $\hat{R}_e(\bm{w}_1^{t},\bm{w}_2^{t})$ (sample mean estimate of the loss over the current batch). For each $e\in \{1,2\}$, say $\bm{w}_{e}^{t}$ is the model used by environment $e$ at the end of iteration $t$.  The two environments alternate to take turns to update the model, i.e. in odd iterations $t$ environment $1$ updates the model, and even iterations $t$ environment $2$ updates the model. Each environment in its turn takes a step based on gradient of its loss over the batch w.r.t its model parameters. In its turn the environment $1$ updates $\bm{w}_1^{t} = \Pi_{\mathcal{W}}[\bm{w}_1^{t-1} - \beta\nabla_{\bm{w}_{1}^{t-1}} \hat{R}_1(\bm{w}_1^{t-1}, \bm{w}_2^{t-1})]$, while the environment $2$ does not update the model, $\bm{w}_2^{t}=\bm{w}_2^{t}-1$, and then $t$ is incremented by $1$. In the next turn, the same procedure is repeated with the roles of environment $1$ and $2$ reversed, i.e., environment $2$ updates and environment $1$ does not. We continue this cycle of updates until a fixed number of epochs.  In our experiments, we set $\beta = 0.005$, the batch size was set to $128$ and the total number of epochs were set to $200$ (each epoch is equal to the size of the training data divide by the batch size). 
 For the implementation of IRM, we needed to change the cross-validation procedure in the implementation provided by \cite{arjovsky2019invariant}. The cross-validation procedure in  requires access to data from a separate validation environment with a different distribution. Since we only use two environments, we use the cross-validation procedure called the train-domain validation set procedure (defined in \cite{gulrajani2020search}), which requires us to split each train environment into a train portion and a validation portion. It finally requires to combine all the validation splits and use them as one validation split. We use a 4:1 split. Besides this change in cross-validation procedure, the rest of the implementation comes from \url{https://github.com/facebookresearch/InvariantRiskMinimization/}.

\subsection{Further details on Figure \ref{fig2}}

Below we provide tables (Table \ref{tab2}, \ref{tab3}, \ref{tab4}, \ref{tab5}) containing numerical values and the standard deviation associated with model estimation error shown in Figure \ref{fig2}. ICP can often be conservative in accepting a covariate as a direct cause, which is the reason we see in some rows the entry against ICP is $5.0 \pm 0.0$; it does not accept any covariate as cause. 

\begin{table}
    \centering
     \renewcommand{\arraystretch}{1.25}
 \begin{tabular}{||c c c ||} 
 \hline
 \textbf{Method} & \textbf{Samples} & \textbf{Error} \\ [0.5ex] 
 \hline\hline
 IRM  &$20$ & $2,72\pm 0.53$  \\ \hline
 ICP  &$20$ & $4.85\pm 0.10$  \\ \hline
ERM            & $20$     & $2.82\pm 0.49$ \\ \hline
C-LRG  &$20$ & $7.96\pm 0.67$  \\ \hline
IRM  &$100$ & $0.96\pm 0.17$  \\ \hline
 ICP  &$100$ & $3.46\pm 0.37$  \\ \hline
ERM            & $100$     & $1.16\pm 0.06$ \\ \hline
C-LRG  &$100$ & $7.98\pm 0.59$  \\ \hline
IRM  &$250$ & $0.59\pm 0.10$  \\ \hline
 ICP  &$250$ & $0.42\pm 0.21$  \\ \hline
ERM            & $250$     & $1.12\pm 0.05$ \\ \hline
C-LRG  &$250$ & $1.11\pm 0.05$  \\ \hline
IRM  &$500$ & $0.90\pm 0.09$  \\ \hline
 ICP  &$500$ & $0.01\pm 0.001$  \\ \hline
ERM            & $500$     & $1.17\pm 0.03$ \\ \hline
C-LRG  &$500$ & $0.65\pm 0.04$  \\ \hline
IRM  &$750$ & $0.54\pm 0.10$  \\ \hline
 ICP  &$750$ & $0.005\pm 0.0001$  \\ \hline
ERM            & $750$     & $1.09\pm 0.03$ \\ \hline
C-LRG  &$750$ & $0.42\pm 0.02$  \\ \hline
IRM  &$1000$ & $0.51\pm 0.11$  \\ \hline
 ICP  &$1000$ & $0.002\pm 0.0003$  \\ \hline
ERM            & $1000$     & $1.12\pm 0.03$ \\ \hline
C-LRG  &$1000$ & $0.43\pm 0.02$  \\ \hline
\end{tabular}
    \caption{\small{Comparisons for F-HET}}
    \label{tab2}
\end{table}

\begin{table}
    \centering
     \renewcommand{\arraystretch}{1.25}
 \begin{tabular}{||c c c ||} 
 \hline
 \textbf{Method} & \textbf{Samples} & \textbf{Error} \\ [0.5ex] 
 \hline\hline
 IRM  &$20$ & $2.48\pm 0.34$  \\ \hline
 ICP  &$20$ & $5.00\pm 0.00$  \\ \hline
ERM            & $20$     & $3.59\pm 0.51$ \\ \hline
C-LRG  &$20$ & $9.31\pm 0.87$  \\ \hline
IRM  &$100$ & $1.28\pm 0.20$  \\ \hline
 ICP  &$100$ & $3.49\pm 0.56$  \\ \hline
ERM            & $100$     & $1.37\pm 0.10$ \\ \hline
C-LRG  &$100$ & $9.86\pm 1.08$  \\ \hline
IRM  &$250$ & $0.95\pm 0.16$  \\ \hline
 ICP  &$250$ & $2.33\pm 0.69$  \\ \hline
ERM            & $250$     & $1.22\pm 0.07$ \\ \hline
C-LRG  &$250$ & $1.23\pm 0.08$  \\ \hline
IRM  &$500$ & $1.01\pm 0.11$  \\ \hline
 ICP  &$500$ & $3.01\pm 0.77$  \\ \hline
ERM            & $500$     & $1.33\pm 0.09$ \\ \hline
C-LRG  &$500$ & $0.89\pm 0.09$  \\ \hline
IRM  &$750$ & $0.85\pm 0.15$  \\ \hline
 ICP  &$750$ & $2.01\pm 0.77$  \\ \hline
ERM            & $750$     & $1.18\pm 0.05$ \\ \hline
C-LRG  &$750$ & $0.53\pm 0.04$  \\ \hline
IRM  &$1000$ & $0.88\pm 0.14$  \\ \hline
 ICP  &$1000$ & $2.50\pm 0.79$  \\ \hline
ERM            & $1000$     & $1.19\pm 0.05$ \\ \hline
C-LRG  &$1000$ & $0.55\pm 0.03$  \\ \hline
\end{tabular}
    \caption{\small{Comparisons for P-HET}}
    \label{tab3}
\end{table}

\begin{table}
    \centering
     \renewcommand{\arraystretch}{1.25}
 \begin{tabular}{||c c c ||} 
 \hline
 \textbf{Method} & \textbf{Samples} & \textbf{Error} \\ [0.5ex] 
 \hline\hline
 IRM  &$20$ & $3.89\pm 0.50$  \\ \hline
 ICP  &$20$ & $5.02\pm 0.02$  \\ \hline
ERM            & $20$     & $4.82\pm 0.57$ \\ \hline
C-LRG  &$20$ & $6.14\pm 0.66$  \\ \hline
IRM  &$100$ & $3.06\pm 0.12$  \\ \hline
 ICP  &$100$ & $7.91\pm 0.46$  \\ \hline
ERM            & $100$     & $4.35\pm 0.12$ \\ \hline
C-LRG  &$100$ & $4.22\pm 0.55$  \\ \hline
IRM  &$250$ & $2.83\pm 0.06$  \\ \hline
 ICP  &$250$ & $7.95\pm 0.47$  \\ \hline
ERM            & $250$     & $4.29\pm 0.12$ \\ \hline
C-LRG  &$250$ & $3.52\pm 0.24$  \\ \hline
IRM  &$500$ & $3.21\pm 0.09$  \\ \hline
 ICP  &$500$ & $6.50\pm 0.58$  \\ \hline
ERM            & $500$     & $4.45\pm 0.05$ \\ \hline
C-LRG  &$500$ & $0.38\pm 0.05$  \\ \hline
IRM  &$750$ & $2.99\pm 0.04$  \\ \hline
 ICP  &$750$ & $5.00\pm 0.00$  \\ \hline
ERM            & $750$     & $4.47\pm 0.04$ \\ \hline
C-LRG  &$750$ & $0.05\pm 0.008$  \\ \hline
IRM  &$1000$ & $3.04\pm 0.06$  \\ \hline
 ICP  &$1000$ & $5.00\pm 0.00$  \\ \hline
ERM            & $1000$     & $4.51\pm 0.07$ \\ \hline
C-LRG  &$1000$ & $0.03\pm 0.003$  \\ \hline
\end{tabular}
    \caption{\small{Comparisons for F-HOM}}
    \label{tab4}
\end{table}

\begin{table}
    \centering
     \renewcommand{\arraystretch}{1.25}
 \begin{tabular}{||c c c ||} 
 \hline
 \textbf{Method} & \textbf{Samples} & \textbf{Error} \\ [0.5ex] 
 \hline\hline
 IRM  &$20$ & $4.03\pm 0.41$  \\ \hline
 ICP  &$20$ & $5.38\pm 0.14$  \\ \hline
ERM            & $20$     & $4.57\pm 0.69$ \\ \hline
C-LRG  &$20$ & $8.03\pm 0.56$  \\ \hline
IRM  &$100$ & $3.39\pm 0.32$  \\ \hline
 ICP  &$100$ & $6.55\pm 0.52$  \\ \hline
ERM            & $100$     & $4.25\pm 0.17$ \\ \hline
C-LRG  &$100$ & $6.24\pm 0.77$  \\ \hline
IRM  &$250$ & $2.95\pm 0.09$  \\ \hline
 ICP  &$250$ & $5.00\pm 0.00$  \\ \hline
ERM            & $250$     & $4.10\pm 0.18$ \\ \hline
C-LRG  &$250$ & $3.27\pm 0.26$  \\ \hline
IRM  &$500$ & $3.02\pm 0.09$  \\ \hline
 ICP  &$500$ & $5.00\pm 0.00$  \\ \hline
ERM            & $500$     & $4.54\pm 0.13$ \\ \hline
C-LRG  &$500$ & $0.56\pm 0.09$  \\ \hline
IRM  &$750$ & $2.81\pm 0.10$  \\ \hline
 ICP  &$750$ & $5.00\pm 0.00$  \\ \hline
ERM            & $750$     & $4.39\pm 0.09$ \\ \hline
C-LRG  &$750$ & $0.08\pm 0.009$  \\ \hline
IRM  &$1000$ & $2.88\pm 0.06$  \\ \hline
 ICP  &$1000$ & $5.00\pm 0.00$  \\ \hline
ERM            & $1000$     & $4.35\pm 0.12$ \\ \hline
C-LRG  &$1000$ & $0.05\pm 0.009$  \\ \hline
\end{tabular}
    \caption{\small{Comparisons for P-HOM}}
    \label{tab5}
\end{table}

\clearpage 

\bibliographystyle{apalike}
\bibliography{LRG_arxiv_jmtd}
\end{document}